%% file: main.tex
\documentclass{article}

\usepackage{arxiv}

\usepackage[utf8]{inputenc} 
\usepackage[T1]{fontenc}    
\usepackage{url}            
\usepackage{booktabs}       
\usepackage{amsfonts}       
\usepackage{nicefrac}       
\usepackage{microtype}      
\usepackage{lipsum}
\usepackage{amsmath}
\usepackage{amsthm}
\usepackage{amssymb}
\usepackage{bigints}
\usepackage{bm}
\usepackage{diagbox}
\usepackage{caption}
\usepackage{subcaption}
\usepackage{graphicx}
\usepackage[ruled,vlined]{algorithm2e}
\include{pythonlisting}
\usepackage{xcolor}
\usepackage{mathtools}
\usepackage{enumitem}

\usepackage{hyperref}
\hypersetup{
    colorlinks=true,
    citecolor=black,
    linkcolor=black,
    filecolor=black,
    urlcolor=black,
}

\newcommand{\R}{\mathbb{R}}

\theoremstyle{definition}

\numberwithin{equation}{section}

\theoremstyle{plain}

\newtheorem{theorem}{Theorem}[section]

\newtheorem{proposition}[theorem]{Proposition}
\newtheorem{lemma}[theorem]{Lemma}

\title{On the eigenvector bias of Fourier feature networks: From regression to solving multi-scale PDEs with physics-informed neural networks}

\author{
  Sifan Wang \\
  Graduate Group in Applied Mathematics \\
  and Computational Science \\
  University of Pennsylvania\\
  Philadelphia, PA 19104 \\
  \texttt{sifanw@sas.upenn.edu} \\
   \And
    Hanwen Wang \\
  Graduate Group in Applied Mathematics \\
  and Computational Science \\
  University of Pennsylvania\\
  Philadelphia, PA 19104 \\
  \texttt{wangh19@sas.upenn.edu} \\
   \And
  Paris Perdikaris \\
  Department of Mechanichal Engineering \\
  and Applied Mechanics\\
  University of Pennsylvania\\
  Philadelphia, PA 19104 \\
  \texttt{pgp@seas.upenn.edu} \\
}

\begin{document}
\maketitle
\begin{abstract}

Physics-informed neural networks (PINNs) are demonstrating remarkable promise in integrating physical models with gappy and noisy observational data, but they still struggle in cases where the target functions to be approximated exhibit high-frequency or multi-scale features. 
In this work we investigate this limitation through the lens of Neural Tangent Kernel (NTK) theory and elucidate how PINNs are biased towards learning functions along the dominant eigen-directions of their limiting NTK. Using this observation, we construct novel architectures that employ spatio-temporal and multi-scale random Fourier features, and justify 
how such coordinate embedding layers can lead to robust and accurate PINN models. Numerical examples are presented for several challenging cases where conventional PINN models fail,  including wave propagation and reaction-diffusion dynamics, illustrating how the proposed methods can be used to effectively tackle both forward and inverse problems involving partial differential equations with multi-scale behavior. All code an data accompanying this manuscript will be made publicly available at \url{https://github.com/PredictiveIntelligenceLab/MultiscalePINNs}.

\end{abstract}







\section{Introduction}
Leveraging advances in automatic differentiation \cite{baydin2017automatic}, 
deep learning tools are introducing a new trend 
in tackling forward and inverse problems in computational mechanics. Under this emerging paradigm, unknown quantities of interest are typically parametrized by deep neural networks, and a multi-task learning problem is posed with the dual goal of fitting observational data and approximately satisfying a given physical law, mathematically expressed via systems of partial differential equations (PDEs). Since the early studies of Psichogios {\em et al.} \cite{psichogios1992hybrid} and Lagaris {\em et al.} \cite{lagaris1998artificial}, and their modern re-incarnation via the framework of physics-informed neural networks (PINNs) \cite{raissi2019physics}, the use of neural networks to represent PDE solutions has undergone rapid growth, both in terms of theory \cite{shin2020convergence, wang2020and, luo2020two, shin2020error} and diverse applications in computational science and engineering \cite{li2020fourier,zhu2019physics,sun2020surrogate}. PINNs in particular, have demonstrated remarkable power in applications including fluid dynamics \cite{raissi2020hidden, jin2020nsfnets, reyes2020learning}, biomedical engineering \cite{sahli2020physics,kissas2020machine,yazdani2020systems}, meta-material design \cite{fang2019deep,chen2020physics}, free boundary problems \cite{wang2020deep}, Bayesian networks and uncertainty quantification \cite{yang2020b,yang2020bayesian}, high dimensional PDEs \cite{han2018solving,raissi2018forward,karumuri2020simulator}, stochastic differential equations \cite{yang2019adversarial}, and beyond \cite{lu2019deeponet, gao2020phygeonet}. However, despite this early empirical success, we are still lacking a concrete mathematical understanding of the mechanisms that render such constrained neural network models effective, and, more importantly, the reasons why these models can oftentimes fail. In fact, more often than not, PINNs are notoriously hard to train, especially for forward problems exhibiting high-frequency or multi-scale behavior. 

Recent work by Wang {\em et al.} \cite{wang2020understanding, wang2020and} has identified two fundamental weaknesses in conventional PINN formulations. The first is related to a remarkable discrepancy of convergence rate between the different terms that define a PINN loss function. As demonstrated by Wang {\em et al.} \cite{wang2020understanding}, the gradient flow of PINN models becomes increasingly stiff for PDE solutions exhibiting high-frequency or multi-scale behavior, often leading to unbalanced gradients during back-propagation. A subsequent analysis using the recently developed neural tangent kernel (NTK) theory \cite{wang2020and}, has revealed how different terms in a PINNs loss may dominate one another, leading to models that cannot simultaneously fit the observed data and minimize the PDE residual. These findings have motivated the development of novel optimization schemes and adaptive learning rate annealing strategies that are demonstrated to be very effective in minimizing multi-task loss functions, such as the ones routinely encountered in PINNs \cite{wang2020understanding, wang2020and}.

The second fundamental weakness of PINNs is related to spectral bias \cite{rahaman2019spectral, cao2019towards, ronen2019convergence}; a commonly observed pathology of deep fully-connected networks that prevents them from learning high-frequency functions. As analyzed in \cite{wang2020and} using NTK theory \cite{jacot2018neural, arora2019exact, lee2019wide}, spectral bias indeed exists in PINN models and is the leading reason that prevents them from accurately approximating high-frequency or multi-scale functions. To this end, recent work in \cite{wang2020multi, li2020dnn, liu2020multi}, attempts to empirically address this pathology by introducing appropriate input scaling factors to  convert the problem of approximating high frequency components of the target function to one of approximating lower frequencies. In another line of work, Tancik {\em et al.} \cite{tancik2020fourier} introduced Fourier feature networks which use a simple Fourier feature mapping to enhance the ability of fully-connected networks to learn high-frequency functions. 
Although these techniques can be effective in some cases, in general, they still lack a concrete mathematical justification in relation to how they potentially address spectral bias.

Building on the these recent findings, this work attempts to analyze and address the aforementioned shortcomings of PINNs, with a particular focus on designing effective models for multi-scale PDEs.
To this end, we rigorously study fully-connected neural networks and PINNs through the lens of their limiting NTK, and produce novel insights into how these models fall short in presence of target functions with high-frequencies of multi-scale features. Using this analysis, we propose a family of novel architectures that can effectively mitigate spectral bias and enable the solution of problems for which current PINN approaches fail. Specifically, our
main contributions can be summarized into the following points:

\begin{itemize}[leftmargin=*]
    \item We argue that spectral bias in deep neural networks in fact corresponds to ``NTK eigenvector bias'', and show that Fourier feature mappings can modulate the frequency of the NTK eigenvectors.
    \item By analyzing how the NTK eigenspace determines the type of functions a neural net can learn, we engineer new effective architectures for multi-scale problems.
    \item  We propose a series of benchmarks for which conventional PINN models fail, and use them to demonstrate the effectiveness of the proposed methods.
\end{itemize}

The remaining of this paper is organized as follows. In section \ref{sec: overview_PINNs}, we present a brief overview of PINNs and emphasize their weakness in solving multi-scale problems. Next, we introduce the neural tangent kernel (NTK) as a theoretical tool to detect and analyze spectral bias in section \ref{sec: NTK_spectrum_bias}. Furthermore, we study the NTK eigensystem of Fourier feature networks and propose two novel network architectures that are efficient in handling multi-scale problems, see section \ref{sec: Fourier_features}, \ref{sec: architecture}.
We present a detailed evaluation of our proposed neural network architectures across a range of representative benchmark examples, see section \ref{sec: results}.
Finally, in section \ref{sec: disscusion}, we summarize our findings and provide a discussion on lingering limitations and promising future directions. 

\section{Physics-informed neural networks}
\label{sec: overview_PINNs}

In this section, we present a brief overview of physics-informed neural networks (PINNs) \cite{raissi2019physics}. In general, we consider partial differential equations of the following form
\begin{align}
\label{eq: PDE}
     &\mathcal{N}[\bm{u}](\bm{x}) = \bm{f}(\bm{x}), \ \  \bm{x} \in \Omega,\\
     \label{eq: BC}
     &\mathcal{B}[\bm{u}](\bm{x})=\bm{g}(\bm{x}), \ \ \bm{x} \in \partial \Omega,
\end{align} 
where $\mathcal{N}[\cdot]$ is a differential operator and  $\mathcal{B}[\cdot] $ corresponds to Dirichlet, Neumann, Robin, or periodic boundary conditions. In addition, $\bm{u}: \overline{\Omega} \rightarrow \R$ describes the unknown latent quantity of interest that is governed by the  PDE system of equation \ref{eq: PDE}. For
time-dependent problems, we consider time $t$ as a special component of $\bm{x}$, and $\Omega$ then also contains the temporal domain. In that case, initial conditions can be simply treated as a special type of boundary condition on the spatio-temporal domain.

Following the original work of Raissi {\em et al.} \cite{raissi2019physics}, we proceed by approximating $\bm{u}(\bm{x})$ by a deep neural network $\bm{u}_{\bm{\theta}}(\bm{x})$, where $\bm{\theta}$ denotes all tunable parameters of the network (e.g., weights and biases). Then, a physics-informed model can be trained by minimizing the following composite loss function
\begin{align}
    \label{eq: PINN_loss}
    \mathcal{L}(\bm{\theta}) = \lambda_r \mathcal{L}_r(\bm{\theta}) + \lambda_b \mathcal{L}_{u_b}(\bm{\theta}),
\end{align}
where 
\begin{align}
    \label{eq: loss_r}
    &\mathcal{L}_r(\bm{\theta}) = \frac{1}{N_r} \sum_{i=1}^{N_r} \left| \mathcal{N}[\bm{u}_{\bm{\theta}}](\bm{x}_r^i) - \bm{f}(\bm{x}_r^i)  \right|^2, \\
    \label{eq: loss_ub}
     &\mathcal{L}_b(\bm{\theta}) = \frac{1}{N_b} \sum_{i=1}^{N_b} \left| \mathcal{B}[\bm{u}_{\bm{\theta}}](\bm{x}_b^i) - \bm{g}(\bm{x}_b^i) \right|^2,
\end{align}
and $N_r$ and $N_b$ denote the batch-sizes of training data  $\{\bm{x}_b^i, \bm{g}(\bm{x}_b^i)\}_{i=1}^{N_b}$ and $\{\bm{x}_r^i, \bm{f}(\bm{x}_r^i)\}_{i=1}^{N_r}$, respectively, which are randomly sampled in the computational domain at each iteration of a gradient descent algorithm. Notice that all required gradients with respect to input variables $\bm{x}$ or parameters $\bm{\theta}$ can be efficiently computed via automatic differentiation \cite{baydin2017automatic}.
Moreover, the parameters $\left\{ \lambda_r,  \lambda_b \right\}$ correspond to weight coefficients in the loss function that can effectively assign a different learning rate to each individual loss term. These weights may be user-specified or tuned automatically during network training \cite{wang2020understanding, wang2020and}. 



Despite a series of early promising results \cite{raissi2020hidden,kissas2020machine,wang2020deep}, the original formulation of Raissi {\em et al.} \cite{raissi2019physics} often struggles to handle multi-scale problems. As an example, let us consider a simple 1D Poisson's equation
\begin{align}
    \label{eq: Poisson1D}
    \Delta u(x) = f(x), \quad x \in (0,1)
\end{align}
subject to the boundary condition
\begin{align*}
    u(0) = u(1) = 0
\end{align*}
Here the fabricated solution we consider is 
\begin{align*}
    u(x) = \sin(2 \pi x) + 0.1 \sin (50 \pi x)
\end{align*}
and $f(x)$ can be derived using equation \ref{eq: Poisson1D}. Though this example is simple and pedagogical, it is worth noting that the solution exhibits low frequency in the macro-scale and high frequency in the micro-scale, which resembles many practical scenarios. 

We represent the unknown solution $u(x)$ by a  5-layer fully-connected neural network $u_{\bm{\theta}(x)}$  with 200 units per hidden layer. The parameters of the network can be learned by minimizing the following loss function 
\begin{align}
    \label{eq: Poisson1D_loss}
    \mathcal{L}(\bm{\theta}) &=  \mathcal{L}_b(\bm{\theta})  + \mathcal{L}_r(\bm{\theta}) \\
    &=\frac{1}{N_b} \sum_{i=1}^{N_b} \left|u_{\bm{\theta}}(x_b^i) - u(x_b^i)  \right|^2 + \frac{1}{N_r}\sum_{i=1}^{N_r} \left| \Delta u_{\bm{\theta}}(x_r^i) - f(x_r^i)  \right|^2
\end{align}
where the batch sizes are set to $N_b = N_r = 128$ and all training  points $\{x_b^i, u(x_b^i)\}_{i=1}^{N_b}$, $\{x_r^i, f(x_r^i)\}_{i=1}^{N_r}$ are uniformly sampled for the boundary and residual collocation points at each iteration of gradient descent.

Figure \ref{fig: Poisson1D_NN} summarized the results obtained  by training the network for $10^7$ iterations of gradient descent using the Adam optimizer \cite{kingma2014adam} with default settings. We observe that the network is incapable of learning the correct solution, even after a million training iterations. In fact, it is not difficult for a conventional fully-connected neural network to approximate that function $u$, given sufficient data inside the computational domain. However, as shown in  figure \ref{fig: Poisson1D_NN}, solving high-frequency or multi-scale problems presents great challenges to PINNs. Although there has been recent efforts to elucidate the reasons why PINN models may fail to train \cite{wang2020understanding, wang2020and}, a complete understanding of how to quantify and resolve such pathologies is still lacking. In the following sections, we will obtain insights by studying Fourier feature networks through  the lens of their neural tangent kernel (NTK), and present a novel methodology to tackle multi-scale problems with PINNs.

\begin{figure}
    \centering
    \includegraphics[width=0.9\textwidth]{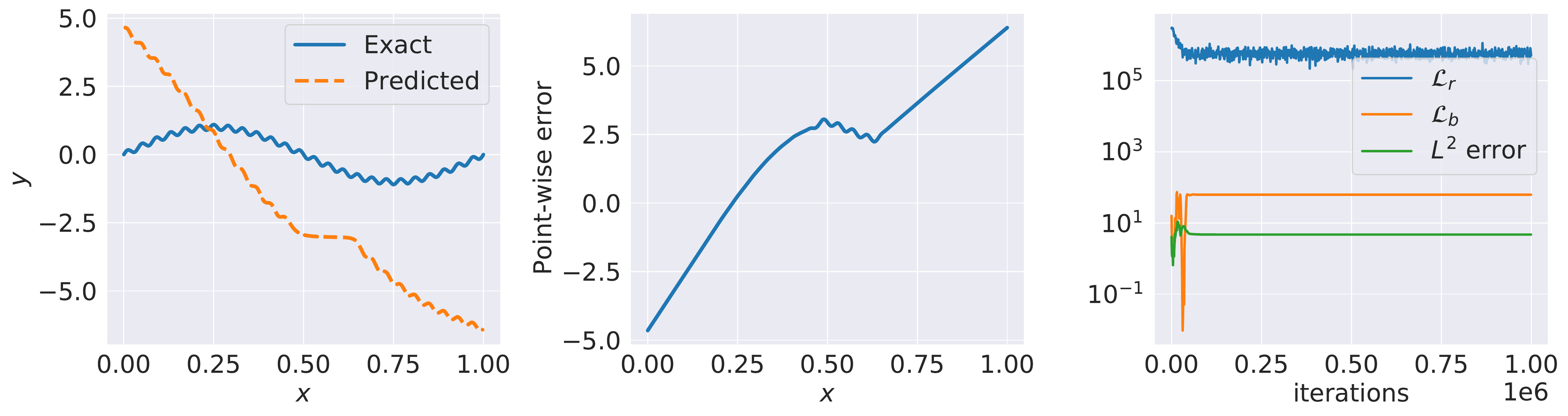}
    \caption{{\em 1D Poisson equation:} Results obtained by training a conventional physics-informed neural network (5-layer, 200 hidden units, $\tanh$ activations) via $10^7$ iterations of gradient descent. 
    {\em Left:} Comparison of the predicted and exact solutions. 
    {\em Middle:} Point-wise error between the predicted and the exact solution. {\em Right:} Evolution of the residual loss $\mathcal{L}_r$, the boundary loss $\mathcal{L}_b$, as well as the relative $L^2$ error  during training. } 
    \label{fig: Poisson1D_NN}
\end{figure}


\section{Methodology} 

\subsection{Analyzing spectral bias through the lens of the Neural Tangent Kernel}
\label{sec: NTK_spectrum_bias}

Before presenting our proposed methods in the context of PINNs, let us first start with a much simpler setting involving regression of functions using deep neural networks. To lay the foundations for our theoretical analysis, we first review the recently developed Neural Tangent Kernel (NTK) theory of Jacot {\em et al.} \cite{jacot2018neural, arora2019exact, lee2019wide},  and its 
connection to investigating spectral bias \cite{rahaman2019spectral, cao2019towards, ronen2019convergence} in the training behavior of deep fully-connected networks.  
Let $f(\bm{x}, \bm{\theta})$ be a scalar-valued fully-connected neural network (see Appendix \ref{appendix: def_FCNN}) with weights $\bm{\theta}$ initialized by a Gaussian distribution $\mathcal{N}(0,1)$. 
Given a data-set $\{\bm{X}_{\text{train}}, \bm{Y}_{\text{train}}\}$, where $\bm{X}_{\text{train}} = (\bm{x}_i)_{i=1}^N$ are inputs and $\bm{Y}_{\text{train}} = (y_i)_{i=1}^N$ are the corresponding outputs, 
we consider a network trained by minimizing the mean square loss $\mathcal{L}(\bm{\theta}) = \frac{1}{N}\sum_{i=1}^N |f(\bm{x}_i, \bm{\theta}) - y_i|^2$ using a very small learning rate $\eta$. Then, following the derivation of Jacot {\em et al.} \cite{jacot2018neural, arora2019exact}, we can define the neural tangent kernel operator $\bm{K}$, whose entries are given by
\begin{align}
    \label{eq: NTK}
   \bm{K}_{ij} = \bm{K}(\bm{x}_i, \bm{x}_j) =  \left\langle\frac{\partial f(\bm{x}_i, \boldsymbol{\theta}))}{\partial \boldsymbol{\theta}}, \frac{\partial f\left(\bm{x}_j , \boldsymbol{\theta}) \right)}{\partial \boldsymbol{\theta}}\right\rangle,
\end{align}
Strikingly, the NTK theory shows that, under gradient descent dynamics with an infinitesimally small learning rate (gradient flow),
the kernel $\bm{K}$ converges to a deterministic kernel $\bm{K}^*$ and does not changes during training as the width of the network grows to infinity. 

Furthermore, under the asymptotic conditions stated in Lee {\em et al.} \cite{lee2019wide}, we can derive that
\begin{align}
    \frac{d f(\bm{X}_{\text{train}}, \bm{\theta}(t)  )}{ d t} \approx - \bm{K} \cdot (f(\bm{X}_{\text{train}}, \bm{\theta}(t)  ) - \bm{Y}_{\text{train}} ),
\end{align}
where $\bm{\theta}(t)$ denotes the parameters of the network at iteration $t$ and $f(\bm{X}_{\text{train}}, \bm{\theta}(t) ) = (f(\bm{x}_i, \bm{\theta}(t) )_{i=1}^N$.  Then, it directly follows that
\begin{align}
     f(\bm{X}_{\text{train}}, \bm{\theta}(t)  ) \approx (I -  e^{-\bm{K}t} ) \cdot \bm{Y}_{\text{train}}.
\end{align}

Since the kernel $\bm{K}$ is positive semi-definite, we can take its spectral decomposition $\bm{K} = \bm{Q}^T\bm{\Lambda} \bm{Q}$, where $\bm{Q}$ is an orthogonal matrix whose $i$-th column is the eigenvector $\bm{q}_i$ of $\bm{K}$ and $\bm{\Lambda}$ is a diagonal matrix whose diagonal entries $\lambda_i$ are the corresponding eigenvalues. 
Since $e^{-\bm{K}t} = \bm{Q^T} e^{-\bm{\Lambda}t} \bm{Q}$, we have 
\begin{align}
    \bm{Q}^T \left(f(\bm{X}_{\text{train}}, \bm{\theta}(t)  ) -       \bm{Y}_{\text{train}} \right) = - e^{\bm{\Lambda}t} \bm{Q}^T \bm{Y_{\text{train}}},
\end{align}
which implies
\begin{align}
      \begin{bmatrix}
        \bm{q}_1^T \\
        \bm{q}_2^T \\
        \vdots \\
        \bm{q}_N^T
    \end{bmatrix}  ( f(\bm{X}_{\text{train}}, \bm{\theta}(t)  ) -       \bm{Y}_{\text{train}}) &=  \begin{bmatrix}
        e^{-\lambda_1 t} & & & \\
        & e^{-\lambda_2 t} & & \\
                    &    &   \ddots &\\
                    &    &  & e^{-\lambda_N t}
    \end{bmatrix}
    \begin{bmatrix}
        \bm{q}_1^T \\
        \bm{q}_2^T \\
        \vdots \\
        \bm{q}_N^T
    \end{bmatrix}
    \bm{Y_{\text{train}}}.
\end{align}
The above equation shows that the convergence rate of $\bm{q}_i^T  ( f(\bm{X}_{\text{train}}, \bm{\theta}(t)  ) - \bm{Y}_{\text{train}}) $ is determined by the $i$-th eigenvalue $\lambda_i$. 
Moreover, we can decompose the training error into the eigenspace of the NTK as
\begin{align}
f(\bm{X}_{\text{train}}, \bm{\theta}(t)  ) -       \bm{Y}_{\text{train}}  &= \sum_{i=1}^N (f(\bm{X}_{\text{train}}, \bm{\theta}(t)  ) -       \bm{Y}_{\text{train}}   , \bm{q}_i) \bm{q}_i \\
&=  \sum_{i=1}^N \bm{q}_i^T \left(f(\bm{X}_{\text{train}}, \bm{\theta}(t)  ) -       \bm{Y}_{\text{train}} \right)\bm{q}_i \\
&= \sum_{i=1}^N \left( e^{-\lambda_i t} \bm{q}_i^T \bm{Y}_{\text{train}} \right)\bm{q}_i.
\end{align}
Clearly, the network is biased to first learn the target function along the eigendirections of neural tangent kernel with larger eigenvalues, and then the rest components corresponding to smaller eigenvalues. 
A more detailed analysis on the convergence rate of different components is illustrated by Cao {\em et al.} \cite{cao2019towards}.
For conventional fully-connected neural networks, the eigenvalues of the NTK shrink monotonically as the frequency of the corresponding eigenfunctions increases, yielding a significantly lower convergence rate for high frequency components of the target function \cite{rahaman2019spectral, ronen2019convergence}.
This indeed reveals the so-called ``spectral bias'' \cite{rahaman2019spectral} pathology of deep neural networks.


Since the learnability of a target function by a neural network can be characterized by the eigenspace of its neural tangent kernel, it is very natural to ask: can we engineer the eigenspace of the NTK to accelerate convergence? If this is possible, can we leverage it to help the network  effectively learn different frequencies in the target function? In the next section, we will answer these questions by re-visitng the recently proposed random Fourier features embedding proposed by Tancik {\em et al.} \cite{tancik2020fourier}.

\subsection{Fourier feature embeddings}

\label{sec: Fourier_features}

Following the original formulation of Tancik {\em et al.} \cite{tancik2020fourier}, a random Fourier mapping $\gamma$ is defined as
\begin{align}
    \gamma(\bm{v})= \begin{bmatrix}
    \cos (\bm{B v} ) \\
    \sin (\bm{Bv} )
    \end{bmatrix},
\end{align}
where each entry in $\bm{B} \in \R^{m \times d}$
is sampled from a Gaussian distribution $\mathcal{N}(0, \sigma^2)$ and $\sigma > 0$ is a user-specified hyper-parameter. 
Then, a Fourier features network \cite{tancik2020fourier} can be simply constructed using a random Fourier features mapping $\gamma$ as a coordinate embedding of the inputs, 
followed by a conventional fully-connected neural network \cite{tancik2020fourier}. 

As shown in \cite{tancik2020fourier}, such a simple method can mitigate the pathology of spectral bias and enable networks to learn high frequencies more effectively, which can significantly improve the effectiveness of neural networks across many tasks including image regression, computed tomography, magnetic resonance imaging (MRI), etc.

In order to explore the deeper reasoning and understand the inner mechanisms behind this simple technique, we consider a two-layer bias-free neural network with Fourier features, i.e.
\begin{align}
    f(\bm{x}) = \frac{1}{\sqrt{m}} \bm{W} \cdot
    \begin{bmatrix}
    \cos (\bm{B x} ) \\
    \sin (\bm{Bx} )
    \end{bmatrix},
\end{align}
where $x \in \R^d$ is the input, $W \in \R^{1 \times 2m}$ is the weight matrix and $\bm{B} = [\bm{b}_1, \bm{b}_2, \dots, \bm{b}_m]^T \in \R^{m \times d}$ are sampled from Gaussian $\mathcal{N}(0, \sigma^2)$. Let $\{\bm{x}_i\}_{i=1}^N$ be input points in a compact domain $C$ . Then, according to equation \ref{eq: NTK}, the neural tangent kernel $\bm{K}$ is given as
\begin{align*}
    \bm{K}_{ij} = \bm{K}(\bm{x}_i, \bm{x}_j) &= \frac{1}{m}  
    \begin{bmatrix}
    \cos (\bm{B} \bm{x}_i)\\
    \sin (\bm{B} \bm{x}_j)
    \end{bmatrix}^{\mathrm{T}}
    \cdot  
        \begin{bmatrix}
    \cos (\bm{B} \bm{x}_i)\\
    \sin (\bm{B} \bm{x}_j)
    \end{bmatrix}  \\
    &= \frac{1}{m} \sum_{k=1}^m \cos(\bm{b}_k^T \bm{x}_i) \cos(\bm{b}_k^T \bm{x}_j)  + \sin(\bm{b}_k^T \bm{x}_i) \sin(\bm{b}_k^T \bm{x}_j)  \\
    &= \frac{1}{m} \sum_{k=1}^m \cos(\bm{b}_k^T (\bm{x}_i - \bm{x}_j)).
\end{align*}
To study the eigen-system of the kernel $\bm{K}$, we consider the limit of $\bm{K}$ as the number of points goes to infinity. In this limit the eigensystem of $\bm{K}$ approaches the eigen-system of the kernel function $K(\bm{x}, \bm{x}')$ which satisfies the following equation \cite{shawe2005eigenspectrum}
\begin{align}
    \label{eq: eigenfunc_equ}
    \int_C K\left(\bm{x}, \bm{x}'\right) g\left(\bm{x}'\right) d \bm{x}'=\lambda g\left(\bm{x}\right), 
\end{align}
where $K({\bm{x}}, {\bm{x}}') = \frac{1}{m} \sum_{k=1}^m \cos(\bm{b}_k^T (\bm{x} - \bm{x}'))$.
Note that  the kernel $K$ induces an Hilbert-Schdmit integral operator $T_K : L^2(C) \rightarrow L^2(C)$ 
\begin{align*}
    T_K(g)(\bm{x}) =  \int_C K\left(\bm{x}, \bm{x}'\right) g\left(\bm{x}'\right) d \bm{x}'.
\end{align*}
Also note that $T_K$
is a compact and self-adjoint operator, which implies that the eigenfunctions exist and all eigenvalues are real.
The following lemma reveals that eigenfunctions are indeed solutions to a eigenvalue problem. 
\begin{lemma}
\label{lemma: eigenfunc}
For the kernel $\bm{K}(\bm{x}, \bm{x}') = \frac{1}{m} \sum_{k=1}^m \cos(\bm{b}_k^T (\bm{x} - \bm{x}'))$, the eigenfunction $g(\bm{x})$ corresponding to non-zero eigenvalues satisfying the the following equation
\begin{align}
        \Delta g(\bm{x}) = - \frac{1}{m} \|\bm{B}\|_F^2  g(\bm{x})
\end{align}
\end{lemma}

\begin{proof}
The proof can be found in Appendix \ref{proof: lemma}.
\end{proof}

If we consider the Laplacian on the sphere $S^{d-1}$ and assume that $\frac{1}{m}\|\bm{B}\|_2^2 = l(l+d-2)$ for some positive integer $l$, then $g(\bm{x})$ are corresponding homogeneous harmonic polynomials of degree $l$ \cite{evans1998partial}. However, in general, directly solving this eigenvalue problem on a complex domain is intractable.

To obtain a better understanding of the behavior of the eigenfunctions and the corresponding eigenvalues,
let us consider a much simper case by setting $d=1$ and $m=1$. Specifically, we take the input $x \in \R$, the compact domain $C = [0,1]$ and the Fourier features
$\bm{B} = b \in \R$ are sampled from a Gaussian distribution $\mathbb{N}(0, \sigma^2)$. Then
the kernel function is given by
\begin{align*}
    K(x, x')  = \cos(b(x - x')).
\end{align*}

In this case, we can compute the exact expression of the  eigenfunctions and their eigenvalues, as summarized in the Proposition \ref{prop: eigenfunc} below.

\begin{proposition}
\label{prop: eigenfunc}
For the kernel function $K(x, x')  = \cos(b(x - x'))$, the non-zero eigenvalues are given by 
\begin{align}
    \lambda =  \frac{1 \pm \frac{\sin b}{b}}{2}.
\end{align}
The corresponding eigenfunctions $g(x)$ must have the form of
\begin{align}
    g(x) = C_1 \cos(b x) + C_2 \sin(b x),
\end{align}
where $C_1$ and $C_2$ are some constants.
\end{proposition}

\begin{proof}
The proof can be found in Appendix \ref{proof: prop_1}.
\end{proof}
From this Proposition, we immediately observe that 
the frequency of the eigenfunctions is determined by $b$ and the gap between the eigenvalues is $\frac{\sin b}{b}$. Besides, recall that $b$ is sampled from a Gaussian distribution $\mathcal{N}(0, \sigma^2)$, which implies that the larger the $\sigma$ we choose, the higher the probability that $b$ takes greater magnitude. Therefore, we may conclude that, for this toy model, large $\sigma$ would lead to high frequency eigenfunctions, as well as narrow eigenvalues gaps. As a result, Fourier features may resolve the issue of spectral bias and enable faster
convergence to high-frequency components of a target function.

Intuitively, we would expect that general fully-connected neural networks with Fourier features exhibit similar behaviors as our toy example. However, it is extremely difficult to  calculate the the eigenvalues and eigenfunctions for generic cases. Therefore, here we attempt to empirically verify our analysis by numerical experiments.

To this end, we first initialize two Fourier feature embeddings with $\sigma = 1, 10$, respectively,
and apply them to a one-dimensional input coordinate before passing them through a 4-layer fully-connected neural network with 100 units per hidden layer. Then, we study the 
NTK eigendecomposition of these two networks at initialization. Figure \ref{fig: reg_sigma_1} and figure \ref{fig: reg_sigma_10} show the visualizations of eigenfunctions and eigenvalues of the NTK computed using  $100$ equally spaced points in $[0,1]$  for  $\sigma = 1, 10$, respectively.  One can observe that the eigenvalues corresponding to the Fourier features with $\sigma = 10$ shrink much slower than the ones corresponding to $\sigma = 1$. Moreover, comparing the eigenvectors for different $\sigma$, it is easy to see that $\sigma = 10$ results in higher frequency eigenvectors than $\sigma = 1$.
This conclusion is further clarified by figure \ref{fig: sigma_spectrum}, which depicts the frequency content of the eigenvector corresponding to the largest eigenvalue, for different $\sigma \in  [1, 50]$. All these observations are consistent with Proposition \ref{prop: eigenfunc} and the analysis presented for the toy network with Fourier features.

\begin{figure}
     \centering
     \begin{subfigure}[b]{0.3\textwidth}
         \centering
         \includegraphics[width=\textwidth]{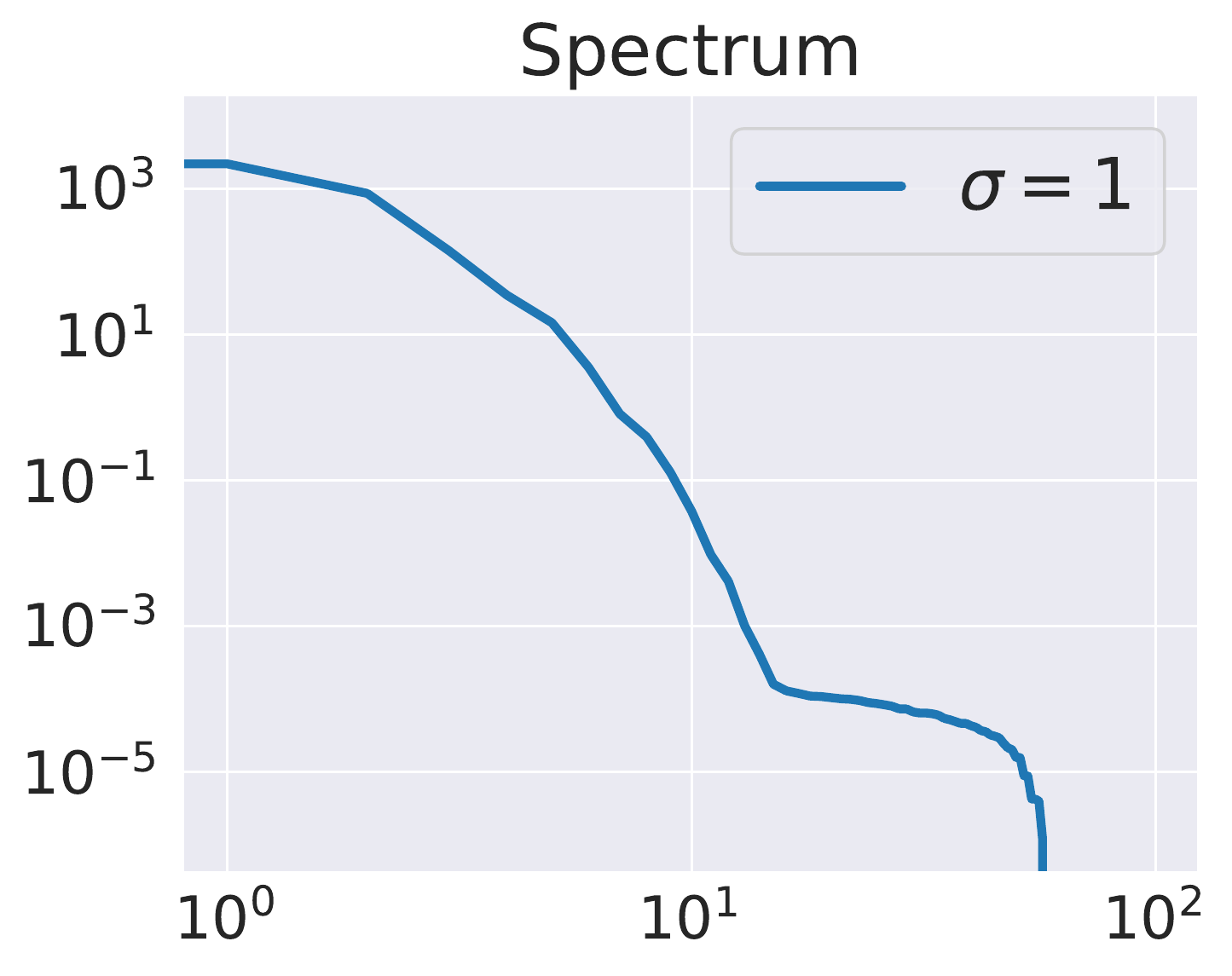}
         \caption{}
         \label{fig: reg_eigenval_sigma_1}
     \end{subfigure}
     \begin{subfigure}[b]{0.5\textwidth}
         \centering
         \includegraphics[width=\textwidth]{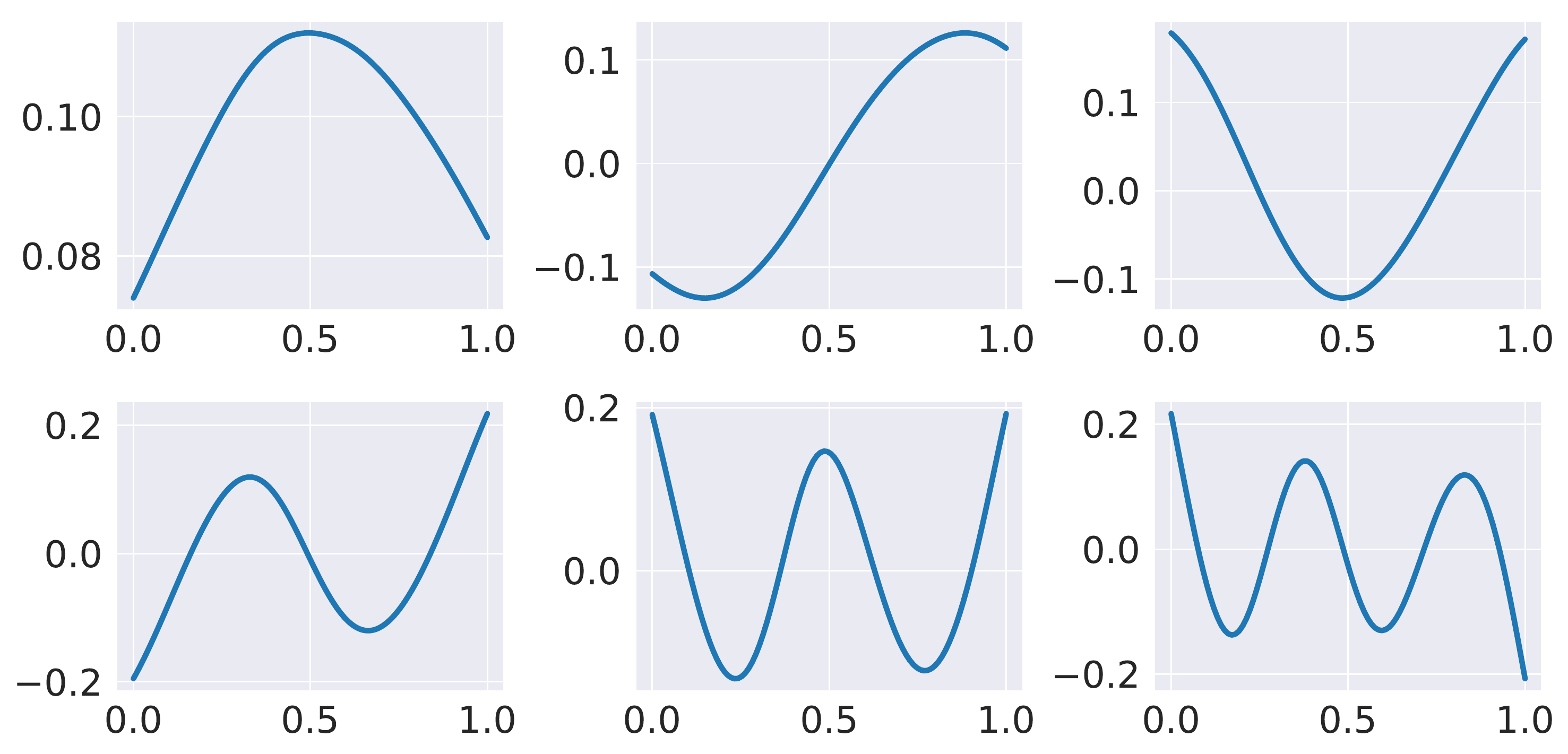}
         \caption{}
         \label{fig: reg_eigenfunc_sigma_1}
     \end{subfigure}
        \caption{{\em NTK eigen-decomposition of a  fully-connected neural network (4 layer, 100 hidden units, $\tanh$ activations) with Fourier features initialized by $\sigma = 1$ on 100 equally spaced points in $[0,1]$:} {\em (a)}: The NTK eigenvalues  in descending order. {\em (b):} The six leading eigenvectors of the NTK in descending order of corresponding eigenvalues.} 
        \label{fig: reg_sigma_1}
\end{figure}

\begin{figure}
     \centering
     \begin{subfigure}[b]{0.3\textwidth}
         \centering
         \includegraphics[width=\textwidth]{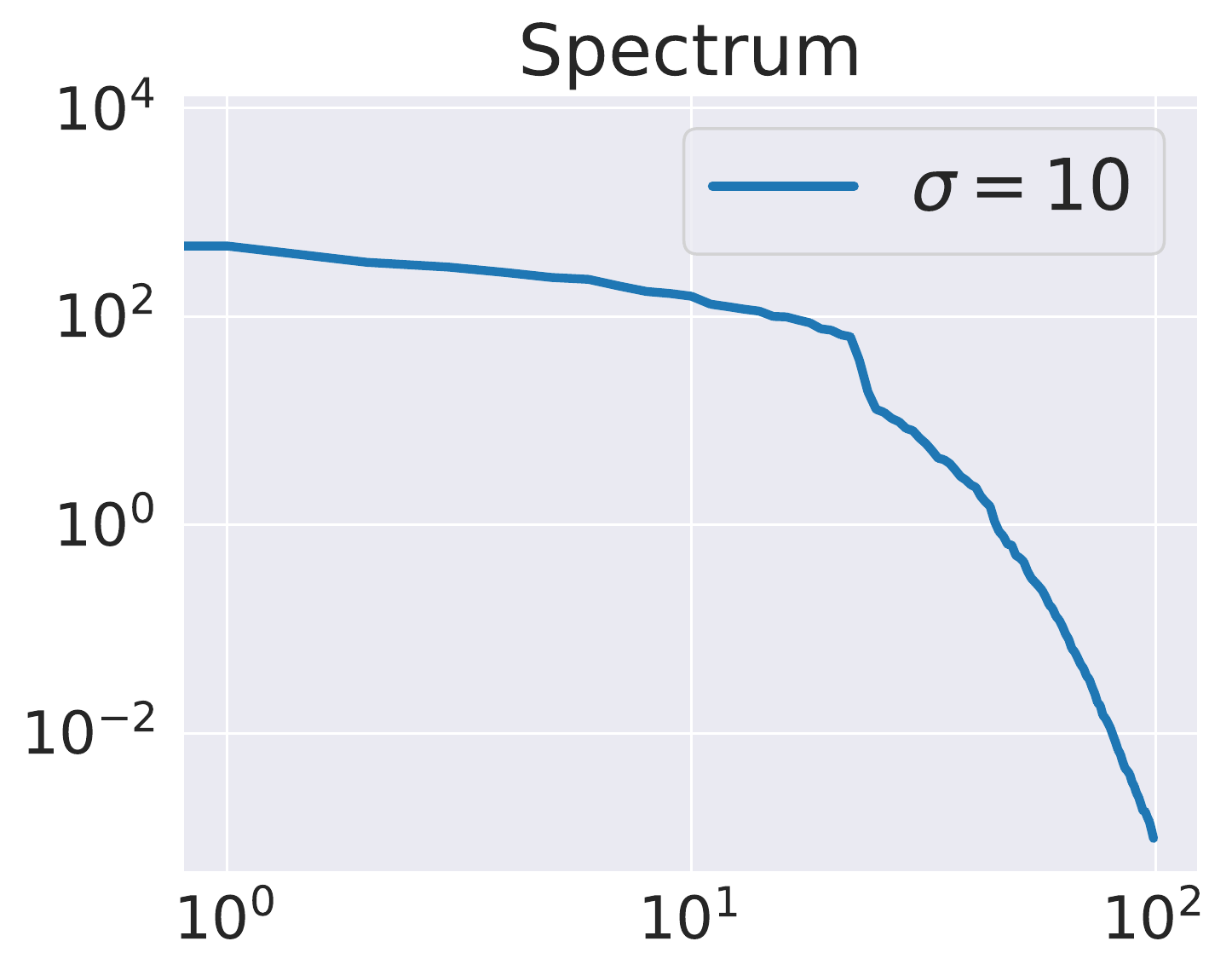}
         \caption{}
         \label{fig: reg_eigenval_sigma_10}
     \end{subfigure}
     \begin{subfigure}[b]{0.5\textwidth}
         \centering
         \includegraphics[width=\textwidth]{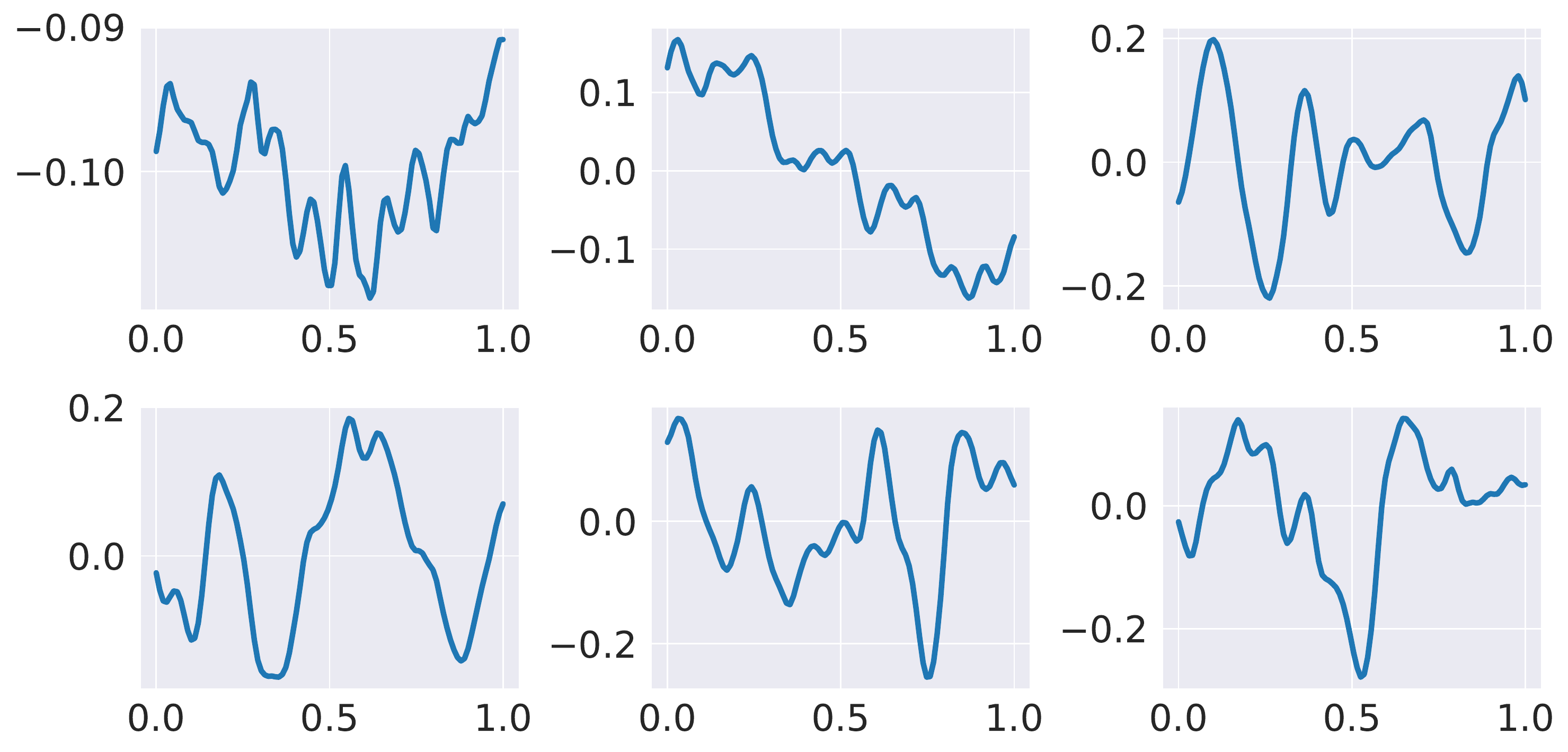}
         \caption{}
         \label{fig: reg_eigenfunc_sigma_10}
     \end{subfigure}
        \caption{{\em NTK eigen-decomposition of a  fully-connected neural network (4 layer, 100 hidden units, $\tanh$ activations) with Fourier features initialized by $\sigma = 10$ on 100 equally spaced points in $[0,1]$:} {\em (a)}: The NTK eigenvalues in descending order. {\em (b):} The six leading eigenvectors of the NTK  in descending order of corresponding eigenvalues.} 
        \label{fig: reg_sigma_10}
\end{figure}

\begin{figure}
    \centering
    \includegraphics[width=0.4\textwidth]{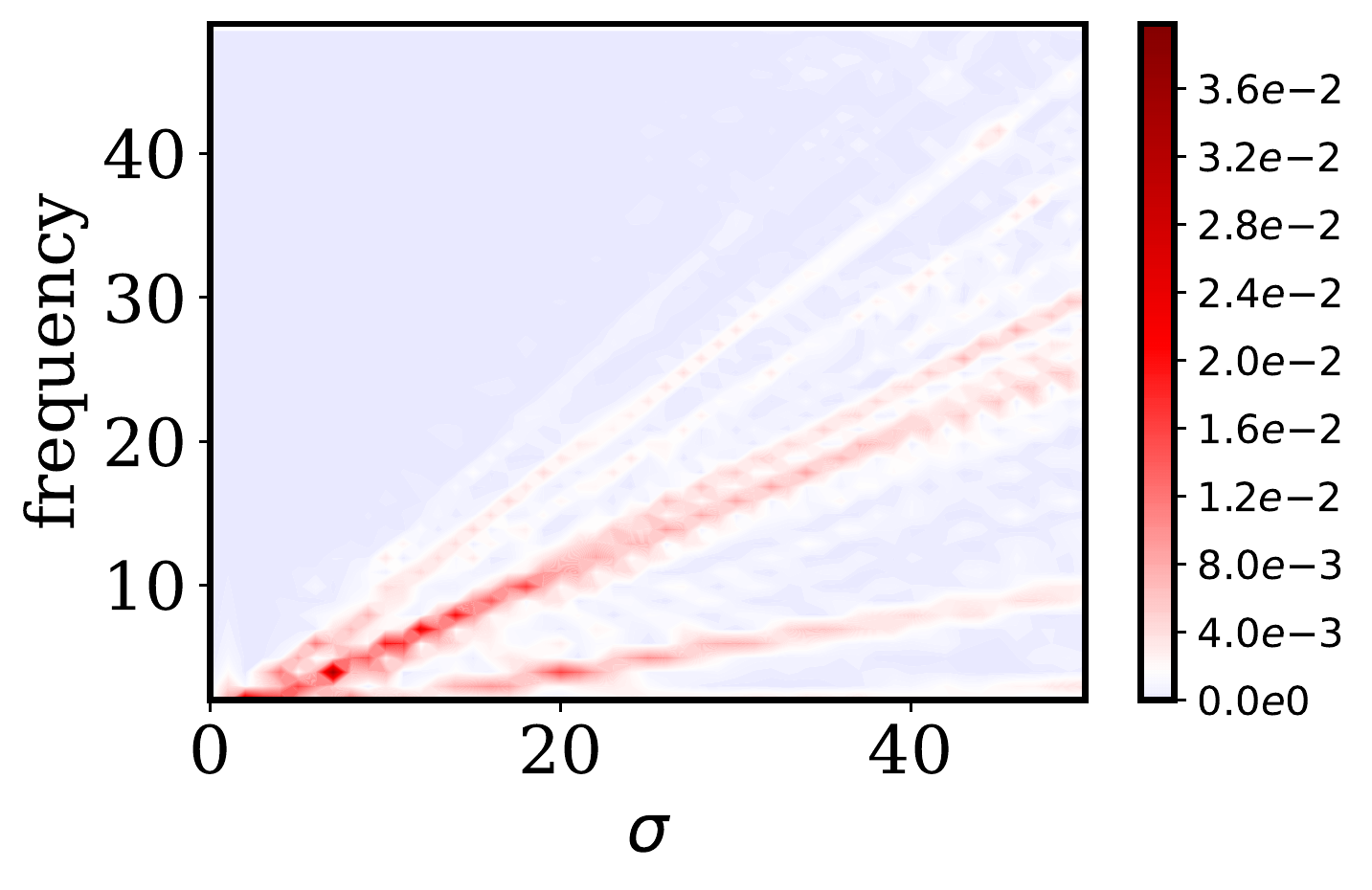}
  \caption{Frequency domain analysis of the first leading NTK eigenvector for a fully-connected neural network (4 layer, 100 hidden units, $\tanh$ activations) with Fourier features initialized by different $\sigma \in [1, 50]$, evaluated on 100 equally spaced points in $[0,1]$.
    }
    \label{fig: sigma_spectrum}
\end{figure}
 
Next, let us consider a simple  one-dimensional target function of the form 
\begin{align}
    f(x) = \sin(20 \pi x) + \sin(2 \pi x), \quad x \in [0,1],
\end{align}
and generate the training data $\{x_i, f(x_i)\}_{i=1}^{100}$ where $x_i$ are  evenly spaced in the unit interval. We proceed by training these two networks to fit the target function using the Adam optimizer \cite{kingma2014adam} with default settings for $1,000$ and $10,000$ epochs respectively.
The results for $\sigma = 1, 10$  are summarized in figure \ref{fig: reg_a_20_sigma_1}
and figure \ref{fig: reg_a_20_sigma_10}, respectively. It can be observed that low frequencies are learned first for $\sigma = 1$, which is pretty similar to the behavior of conventional fully-connected neural networks, commonly referred to as ``spectral bias'' \cite{rahaman2019spectral}. 
Notice, however, that it is high frequencies that are learned first for  $\sigma = 10$. As shown in figure \ref{fig: reg_sigma_1}
 and figure \ref{fig: reg_sigma_10}, we already know that the value of $\sigma$ determines the frequency of eigenvectors of the NTK. Therefore, this observation highly suggests that ``spectral bias'' actually corresponds to  ``eigenvector bias'', in the sense that the leading eigenvectors corresponding to large eigenvalues of the NTK determine the frequencies which the network is biased to learn first. 

Furthermore, one may note that the distribution of the eigenvalues corresponding to $\sigma = 1$ moves ``outward'' during training.  From our experience, the movement of the eigenvalue distribution results in the movement of its NTK, as well as the parameters of the network during training.
As shown in figure \ref{fig: reg_a_20_sigma_1_spec_error}), the parameters of the network barely move after a rapid change in the first few hundred epochs.
This indicates that the network initialization  is not suitable to fit the given target function because the parameters of the network have to move very far from initialization in order to reach a reasonable local minimum. In contrast, the distribution of the the eigenvalues corresponding to $\sigma = 10$ almost keeps the NTK spectrum fixed during training. Accordingly, in the middle panel of figure \ref{fig: reg_a_20_sigma_10_spec_error}, similar behavior can be observed, but the relative change of the parameters for case of $\sigma = 10$ is much less than the case of $\sigma=1$.
This suggests that the initialization of the network is ``good'' and desirable local minima exist in the vicinity of the parameter initialization in the corresponding loss landscape.
As shown in the top right panels of figure \ref{fig: reg_a_20_sigma_1} and figure \ref{fig: reg_a_20_sigma_10}, the relative $L^2$ prediction error corresponding to $\sigma =10$ decreases much faster than the relative $L^2$ error corresponding to $\sigma =1$. 

Finally, it is worth emphasizing that Fourier feature mappings initialized by large $\sigma$ do not always benefit the network, as too large value of $\sigma$ may cause over-fitting. To demonstrate this point, we initialize a Fourier feature mapping with $\sigma = 10$ and pass it through the same fully-connected network. We now consider $f(x) = \sin{\pi x} + \sin(2 \pi x)$ as the ground truth target function, from which $20$ points are uniformly sampled as training data. Then we train the network to fit the target function using the Adam optimizer with default settings for $1,000$ epochs. As shown in figure \ref{fig: reg_a_1_sigma_10}, although all training data pairs are perfectly approximated and the training error is very small, the network interpolates the training data with high frequency oscillations, and thus fails to correctly recover the target function. One possible explanation is that 
the neural network approximation tends to exhibit similar frequencies as the leading eigenvectors of its NTK.
Therefore, choosing an appropriate $\sigma$ such that the frequency of the leading NTK eigenvectors agrees with the  frequency of the target function plays an important role in the network performance, which not only accelerates the convergence speed, but also increases the prediction accuracy.

\begin{figure}
     \centering
     \begin{subfigure}[b]{0.8\textwidth}
         \centering
         \includegraphics[width=\textwidth]{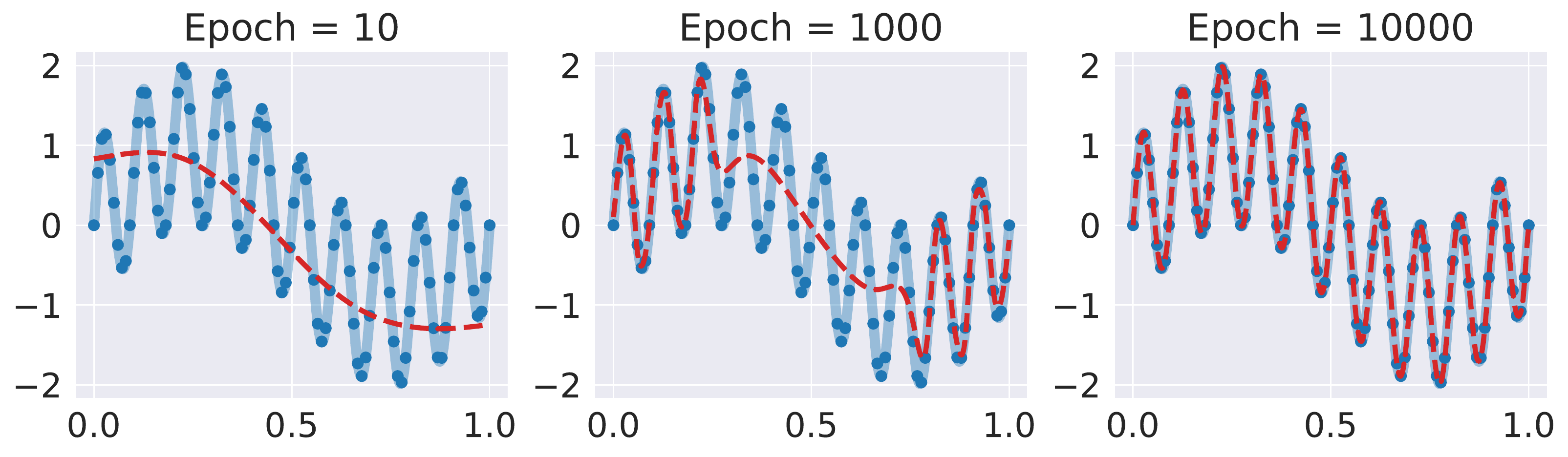}
         \caption{}
         \label{fig: reg_a_20_sigma_1_func}
     \end{subfigure}
     \begin{subfigure}[b]{0.8\textwidth}
         \centering
         \includegraphics[width=\textwidth]{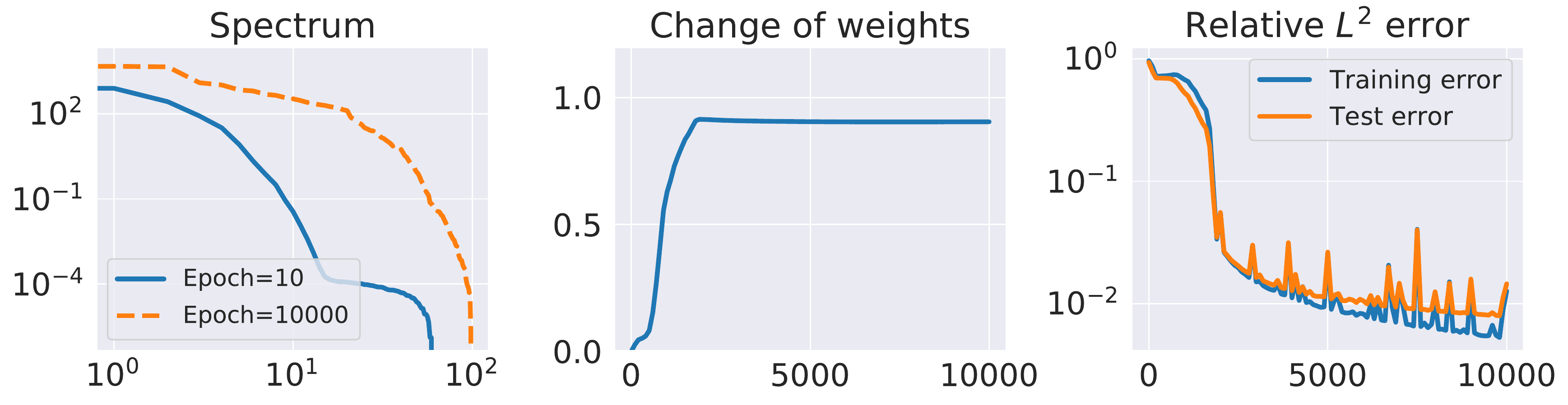}
         \caption{}
         \label{fig: reg_a_20_sigma_1_spec_error}
     \end{subfigure}
         \caption{{\em Training a network with Fourier features initialized by $\sigma = 10$ to fit the target function $f(x) = \sin(20\pi x) + \sin(2 \pi x)$ for $10,000$ epochs:} 
         {\em (a):} Network prediction (dash red) against the ground truth (light blue).  The network prediction exhibits high frequencies when fitting the data points during training.
         {\em (b) middle:} Relative change of the parameters $\bm{\theta}$ ($\frac{||\theta(t) - \theta(0)||_2}{\||\theta(0)\||_2}$) of the network during training.
         {\em (b) right:}  Relative $L^2$ training error and test error during training.}
        \label{fig: reg_a_20_sigma_1}
\end{figure}

\begin{figure}
     \centering
     \begin{subfigure}[b]{0.8\textwidth}
         \centering
         \includegraphics[width=\textwidth]{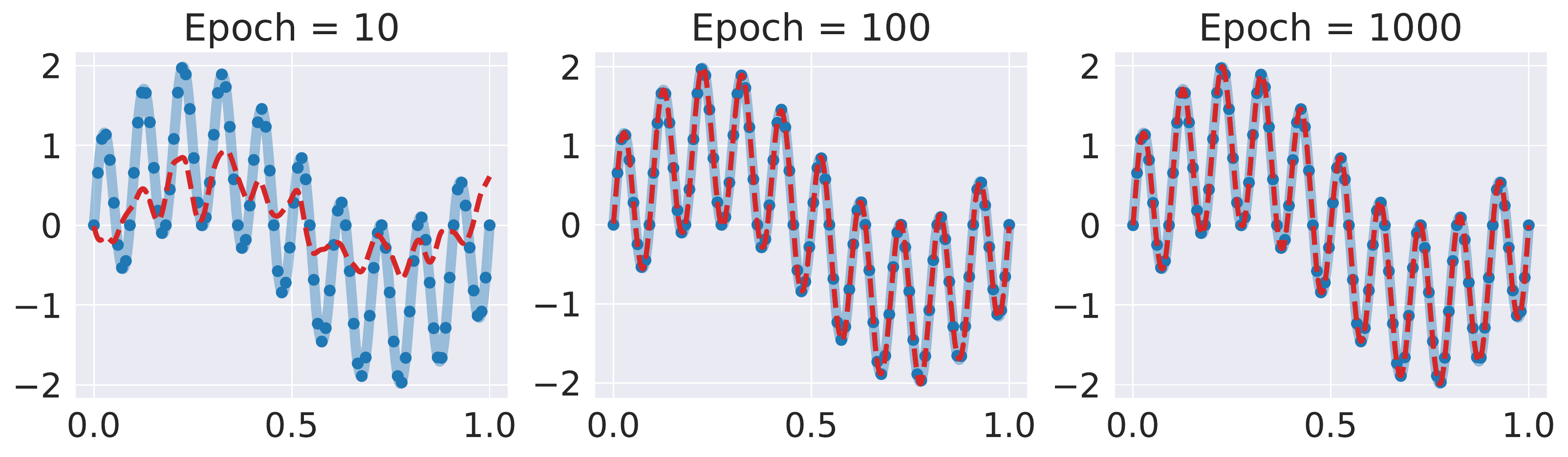}
         \caption{}
         \label{fig: reg_a_20_sigma_10_func}
     \end{subfigure}
     \begin{subfigure}[b]{0.8\textwidth}
         \centering
         \includegraphics[width=\textwidth]{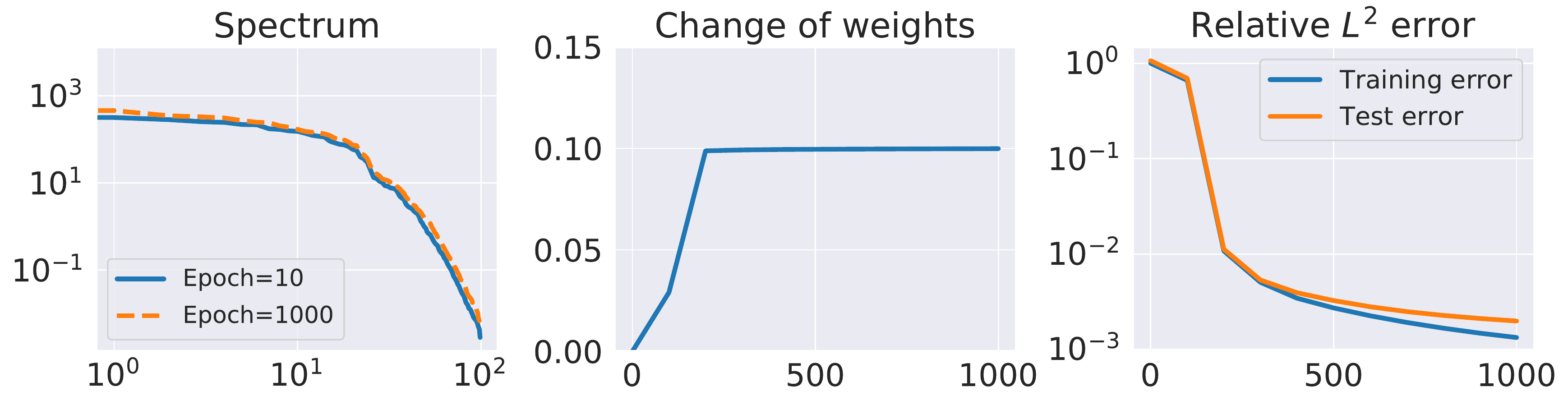}
         \caption{}
         \label{fig: reg_a_20_sigma_10_spec_error}
     \end{subfigure}
         \caption{{\em Training a network with Fourier features initialized by $\sigma = 10$ to fit the target function $f(x) = \sin(20\pi x) + \sin(2 \pi x)$ for $1,000$ epochs:} 
         {\em (a):} Network prediction (dash red) against the ground truth (light blue).  The network prediction exhibits high frequencies when fitting the data points during training.
         {\em (b) left:} Evolution of NTK eigenvalues during training. 
         {\em (b) middle:} Relative change of the parameters $\bm{\theta}$ ($\frac{||\theta(t) - \theta(0)||_2}{\||\theta(0)\||_2}$) of the network during training.
         {\em (b) right:} Relative $L^2$ training error and test error during training. }
        \label{fig: reg_a_20_sigma_10}
\end{figure}

\begin{figure}
     \centering
     \begin{subfigure}[b]{0.8\textwidth}
         \centering
         \includegraphics[width=\textwidth]{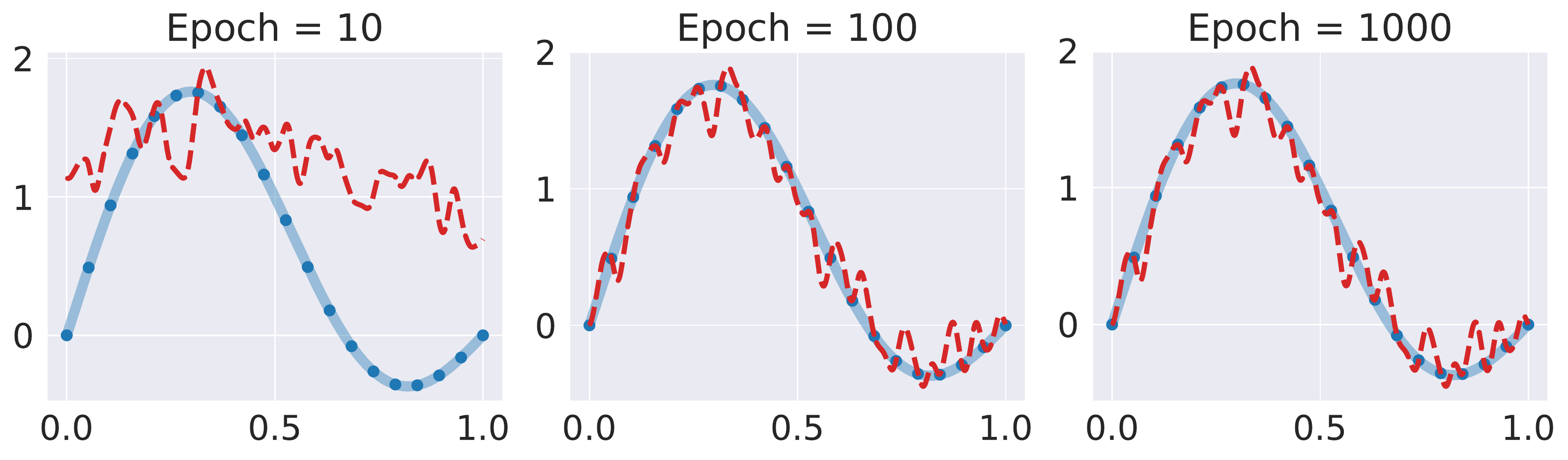}
         \caption{}
         \label{fig: reg_a_1_sigma_10_func}
     \end{subfigure}
     \begin{subfigure}[b]{0.8\textwidth}
         \centering
         \includegraphics[width=\textwidth]{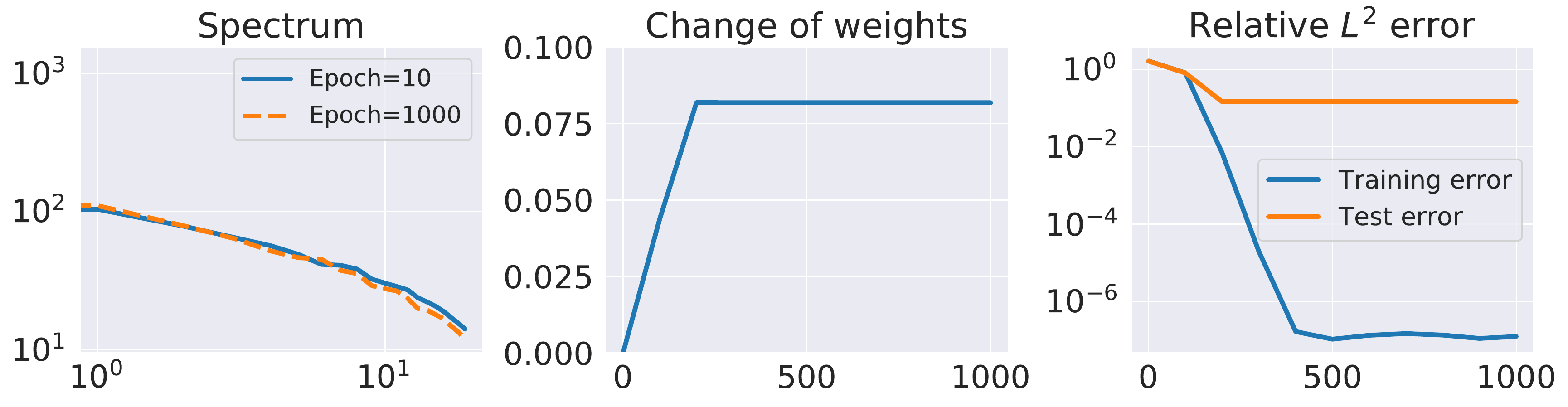}
         \caption{}
         \label{fig: reg_a_1_sigma_10_spec_error}
     \end{subfigure}
         \caption{{\em Training a network with Fourier features initialized by $\sigma = 10$ to fit the target function $f(x) = \sin(\pi x) + \sin(2 \pi x)$ for 1,000 epochs:} 
         {\em (a):} Network prediction (dash red) against the ground truth (light blue).  The network prediction exhibits high frequencies when fitting the data points during training.
           {\em (b) middle:} Relative change of the parameters $\bm{\theta}$ ($\frac{||\theta(t) - \theta(0)||_2}{\||\theta(0)\||_2}$) of the network during training.
         {\em (b) left:} Evolution of NTK eigenvalues  during training. {\em (b) right:} Relative $L^2$ training error and test error during training. }
        \label{fig: reg_a_1_sigma_10}
\end{figure}

\subsection{Multi-scale Fourier feature embeddings for physics-informed neural networks}

\label{sec: architecture}






In the previous section, we presented a detailed theoretical and numerical analysis of the NTK eigen-system of a neural network for classical $\ell_2$ regression problems. Building on this insight, we now draw our attention back to physics-informed neural networks for solving forward and 
inverse problems involving partial differential equations, whose solutions may exhibit multi-scale behavior.

To begin, we first note that the NTK of PINNs \cite{wang2020and} is slightly more complicated than the NTK of networks used for conventional regressions. To this end, we follow the setup presented in Wang {\em et al.} \cite{wang2020and}, and consider general PDEs with appropriate boundary conditions (see equations \ref{eq: PDE} and \ref{eq: BC}) and corresponding ``training data'' $\{\bm{x}_b^i, g(\bm{x}_b^i)\}_{i=1}^{N_b}, \{\bm{x}_r^i, f(\bm{x}_r^i)\}_{i=1}^{N_r}$, and define the NTK of PINNs as
\begin{align*}
    \bm{K}(t) =  \begin{bmatrix}
     \bm{K}_{uu}(t) & \bm{K}_{ur}(t) \\
     \bm{K}_{ru}(t) & \bm{K}_{rr}(t)
    \end{bmatrix},
\end{align*}
where $\bm{K}_{ru}(t) = \bm{K}_{ur}^T(t)$ and
$\bm{K}_{uu}(t) \in \R^{N_b \times N_b}, \bm{K}_{ur}(t) \in \R^{N_b \times N_r}, \text{ and } \bm{K}_{rr}(t) \in \R^{N_r \times N_r}$, whose $(i,j)$-th entry is given by
\begin{align}
       & (\bm{K}_{uu})_{ij}(t) =  \Big\langle  \frac{d \mathcal{B}[\bm{u}]({\bm x}_b^i, {\bm \theta}(t))}{d{\bm \theta}},  \frac{d\mathcal{B}[\bm{u}](\bm{x}_b^j, {\bm \theta}(t))}{d{\bm \theta}}    \Big\rangle \\
   & (\bm{K}_{ur})_{ij}(t)  =  \Big\langle  \frac{d \mathcal{B}[\bm{u}](\bm{x}_b^i, {\bm \theta}(t))}{d{\bm \theta}} , \frac{d \mathcal{N}[\bm{u}](\bm{x}_r^j, {\bm \theta}(t))}{d{\bm \theta}}  \Big\rangle  \\
   & (\bm{K}_{rr})_{ij}(t)  =  \Big\langle \frac{d \mathcal{N}[\bm{u}](\bm{x}_r^i, {\bm \theta}(t))}{d{\bm \theta}},  \frac{d \mathcal{N}[\bm{u}](\bm{x}_r^j, {\bm \theta}(t))}{d{\bm \theta}}   \Big\rangle.
\end{align}

Then, the training dynamics of PINNs under gradient descent with an infinitesimally small learning rate can be characterized by the following ODE system
\begin{align}
\label{eq: PINN_ode}
    \begin{bmatrix}
    \frac{d \mathcal{B}[\bm{u}] (\bm{x}_b, {\bm \theta}(t))}{dt}\\
    \frac{d \mathcal{N}[\bm{u}](\bm{x}_r, {\bm \theta}(t))}{dt}
    \end{bmatrix}
    =
       - \begin{bmatrix}
     \bm{K}_{uu}(t) & \bm{K}_{ur}(t) \\
     \bm{K}_{ru}(t) & \bm{K}_{rr}(t)
    \end{bmatrix}
    \cdot
       \begin{bmatrix}
    \mathcal{B}[\bm{u}](\bm{x}_b, {\bm \theta}(t)) - \bm{g}(\bm{x}_b) \\
    \mathcal{N}[\bm{u}](\bm{x}_r, {\bm \theta}(t)) - \bm{f}(\bm{x}_r)
    \end{bmatrix}.
\end{align}

Then, the NTK framework allows us to show the following Proposition.

\begin{proposition}  
\label{prop: pinns}
Suppose that the training dynamics of PINNs satisfies equation \ref{eq: PINN_ode} and the spectral decompositions of $\bm{K}_{uu}(0)$ and $\bm{K}_{rr}(0)$ are
\begin{align}
    &\bm{K}_{uu}(0) =  \bm{Q}_u^T \Lambda_u  \bm{Q}_u^T \\
    &\bm{K}_{rr}(0) =  \bm{Q}_r^T \Lambda_r  \bm{Q}_r^T, 
\end{align}
where $\bm{Q}_u$ and $\bm{Q}_r$ are orthogonal matrices consisting of eigenvectors of $\bm{K}_{uu}(0)$ and $\bm{K}_{rr}(0)$, respectively, and
$\Lambda_u$ and $\Lambda_r$ are diagonal matrices whose entries are the eigenvalues of $\bm{K}_{uu}(0)$ and $\bm{K}_{rr}(0)$, respectively.
Given the assumptions
\begin{enumerate}[label=(\roman*),leftmargin=*]
    \item  $\bm{K}(t) \approx \bm{K}(0)$ for all $t \geq 0$.

    \item $\bm{K}_{uu}(0)$ and $\bm{K}_{rr}(0)$ are positive definite,
\end{enumerate}
we can write $\bm{B} = \bm{Q}_r^T, \bm{K}_{ru}(0) \bm{Q}_u$, and obtain
\begin{align}
\bm{Q}^T 
     \left( \begin{bmatrix}
    \mathcal{B}[\bm{u}](\bm{x}_b, {\bm \theta}(t))\\
   \mathcal{N}[\bm{u}](\bm{x}_r, {\bm \theta}(t))
    \end{bmatrix} -
     \begin{bmatrix}
     \bm{g}(\bm{x}_b) \\
    \bm{f}(\bm{x}_r)
    \end{bmatrix} \right) 
    \approx e^{- \bm{P }^T  \Tilde{\bm{\Lambda}}  \bm{P } t}\bm{Q}^T 
     \begin{bmatrix}
      \bm{g}(\bm{x}_b) \\
       \bm{f}(\bm{x}_r)
       \end{bmatrix},
\end{align}
where  
\begin{align}
    \bm{Q} = \begin{bmatrix}
             \bm{Q}_u & 0\\
           0 &  \bm{Q}_r
           \end{bmatrix}, \quad
\bm{P} = \begin{bmatrix}
           \bm{I} & 0 \\
           - \bm{B} \bm{\Lambda}_u^{-1}  & \bm{I}
            \end{bmatrix}, \quad
        \bm{\Lambda} =  \begin{bmatrix}
           \bm{ \Lambda}_u & 0\\
           0& \bm{\Lambda}_r - \bm{B}^T \bm{\Lambda}_u^{-1} \bm{B}
           \end{bmatrix}.
\end{align}

\end{proposition}

The above proposition reveals that, under some assumptions, the resulting NTK eigen-system of PINNs is determined by the eigenvectors of $\bm{K}_{uu}$ and $\bm{K}_{rr}$. Understanding the behavior of this eigen-system derived from the NTK of PINNs should be at the core of future extensions of this line of research.

As mentioned in section \ref{sec: overview_PINNs}, PINNs often struggle in solving multi-scale problems. Unlike  conventional regression tasks, there is generally no or just a handful data points provided for PINNs inside the computational domain. This is similar to the case illustrated in figure \ref{fig: reg_a_1_sigma_10} where the network would fit the target function biases towards its preferred frequencies, which are determined by the eigenvectors of its NTK. 
Consequently, PINNs using fully-connected networks would learn the solutions and their PDE residuals with the lowest frequency first, due to ``spectral bias''. We believe that this may be one of the fundamental reasons that cause failure of PINNs in learning high-frequency or multi-scale solutions of PDEs. 

Inspired by our analysis and observations of Fourier features in section \ref{sec: Fourier_features}, we present a novel network architecture to handle multi-scale problems. As illustrated in figure \ref{fig: arch_mFF}, we apply multiple Fourier feature embeddings initialized with different $\sigma$ to input coordinates before passing these embedded inputs through the same fully-connected neural network and finally concatenate the outputs with a linear layer.  The detailed forward pass is defined as follows:
\begin{align}
    &\bm{\gamma}^{(i)}(\bm{x}) = 
    \begin{bmatrix}
        \cos (2 \pi \bm{B}^{(i)} \bm{x})\\
        \sin (2 \pi \bm{B}^{(i)} \bm{x})
    \end{bmatrix},  \quad \text{ for } i=1, 2, \dots, M\\
    &\bm{H}^{(i)}_1 = \phi(\bm{W}_1 \cdot \bm{\gamma}^{(i)}(\bm{x})  + \bm{b}_1),  \quad \text{ for } i=1, 2, \dots, M \\
    & \bm{H}^{(i)}_\ell = \phi(\bm{W}_\ell \cdot \bm{H}^{(i)}_{\ell - 1}  + \bm{b}_\ell),  \quad \text{ for } \ell=2,  \dots, L,  i=1, 2, \dots, M\\
        & \bm{f}_{\bm{\theta}}(\bm{x}) = \bm{W}_{L+1} \cdot \left[  \bm{H}^{(1)}_L,  \bm{H}^{(2)}_L, \dots,   \bm{H}^{(M)}_L  \right] + \bm{b}_{L+1},
\end{align}
where $\bm{\gamma}^{(i)}$ and $\phi$ denote Fourier feature mappings and activation functions, respectively, and each entry in $\bm{B}^{(i)} \in \R^{m \times d}$ is sampled from $\mathcal{N}(0, \sigma_i)$, and is held fixed during model training (i.e. $\bm{B}^{(i)}$ are not trainable parameters, as in \cite{tancik2020fourier}). 
Notice that the weights and the biases of this architecture are essentially the same as in a standard fully-connected neural network.
Here, we underline that the choice of $\sigma_i$ is problem-dependent and typical values can be  $1, 20, 50, 100, $ etc. 
%

To better understand the motivation behind this architecture, suppose that $\bm{f}_{\bm{\theta}}$ is a approximation of a given target function $f$ whose Fourier decomposition is 
\begin{align}
    \label{eq: Fourier_decomposition}
  f(x)=\sum_{k=-\infty}^{\infty} \hat{f}_{k} e^{i k x},
\end{align}
where $\hat{f}_k$ is the Fourier coefficient corresponding to the wave-number $k$ \cite{hesthaven2007spectral}. Note that $\bm{f}_{\bm{\theta}}$ is simply a linear combination of $\{\bm{H}_L^{(i)}\}_{i=1}^M$, which has some degree of consistency with equation \ref{eq: Fourier_decomposition}.
Moreover, we emphasize again that, for networks with Fourier features,  $\sigma$ determines the frequency that the networks prefer to learn. As a result, if we just employ networks with one Fourier feature embedding,
then whatever the value of $\sigma$ is used to initialize Fourier feature mappings, will yield slower convergence to the rest frequency components, except for the preferable frequencies determined by the choice of $\sigma$. Therefore, it is reasonable to embed inputs to several Fourier feature mappings with different $\sigma$ and concatenate them through a linear layer after the forward propagation such that all frequency components can be learned with the same convergence rate. 






For time-dependent problems, multi-scale behavior may exist not only across spatial directions but also across time. Thus, we present another novel multi-scale Fourier feature architecture to tackle multi-scale problems in spatio-temporal domains. Specifically, the feed-forward pass of the network is now defined as


\begin{align}
     &\bm{\gamma}^{(i)}_{\bm{x}}(\bm{x}) = 
    \begin{bmatrix}
        \cos (2 \pi \bm{B}^{(i)}_{\bm{x}} \bm{x})\\
        \sin (2 \pi \bm{B}^{(i)}_{\bm{x}} \bm{x})
    \end{bmatrix}, \ 
    \bm{H}^{(i)}_{\bm{x}, 1} = \phi(\bm{W}_1 \cdot \bm{\gamma}^{(i)}_{\bm{x}}(\bm{x})  + \bm{b}_1),
    \quad \text{ for } i=1, 2, \dots, M_x,\\
    &\bm{\gamma}^{(j)}_{t}(t) = 
    \begin{bmatrix}
        \cos (2 \pi \bm{B}^{(j)}_{t} \bm{x})\\
        \sin (2 \pi \bm{B}^{(j)}_{t} t)
    \end{bmatrix}, \ 
    \bm{H}^{(j)}_{t, 1} = \phi(\bm{W}_1 \cdot \bm{\gamma}^{(j)}_{t}(t)  + \bm{b}_1), 
    \quad \text{ for } j=1, 2, \dots, M_t,\\
    \label{eq: ST_mFF_forward_pass_1}
    & \bm{H}^{(i)}_{\bm{x}, \ell} = \phi(\bm{W}_\ell \cdot \bm{H}^{(i)}_{\bm{x}, \ell-1}  + \bm{b}_\ell),  \quad \text{ for } \ell=2,  \dots, L \text{ and } i=1,2, \dots, M_x,\\
     \label{eq: ST_mFF_forward_pass_2}
    & \bm{H}^{(j)}_{t, \ell} = \phi(\bm{W}_\ell \cdot \bm{H}^{(j)}_{t, \ell-1}  + \bm{b}_\ell),  \quad \text{ for } \ell=2,  \dots, L \text{ and } j=1,2, \dots, M_t,\\
     \label{eq: ST_mFF_concatenate_1}
    &    \bm{H}_{L}^{(i,j)} =  \bm{H}^{(i)}_{\bm{x}, L} \odot  \bm{H}^{(j)}_{t, L}, \quad \text{ for } i=1,2, \dots, M_x \text{ and } j=1,2, \dots, M_t, \\
    \label{eq: ST_mFF_concatenate_2}
    & \bm{f}_{\bm{\theta}}(\bm{x}, t) = \bm{W}_{L+1} \cdot \left[  \bm{H}_{L}^{(1,1)}, \dots,  \bm{H}_{L}^{(M_x,M_t)} 
\right] + \bm{b}_{L+1},
\end{align}
where $\bm{\gamma}_{\bm{x}}^{(i)}$ and $\bm{\gamma}_t^{(j)}$ denote spatial and temporal Fourier feature mappings, respectively, and $\odot$ represents the point-wise multiplication. Here each entry of
$\bm{B}_{\bm{x}}^{(i)}$ and $\bm{B}_t^{(j)}$ are sampled from $\mathcal{N}(0, \sigma_i^{\bm{x}})$ and  $\mathcal{N}(0, \sigma_j^t)$, respectively, and are held fixed during model training.  A visualization of this architecture is presented in figure \ref{fig: arch_spatail_temporal_mFF}. One key difference from figure \ref{fig: arch_mFF}
is that we apply separate Fourier feature embeddings to spatial and temporal input coordinates before passing the embedded inputs through the same fully-connected network. Another key difference is that we merge spatial outputs $\bm{H}_{\bm{x}, L}^{(i)}$ and temporal outputs $\bm{H}_{t, L}^{(j)}$ using point-wise multiplication and passing them through a linear layer. Heuristically, this architecture is consistent with the Fourier spectral method \cite{hesthaven2007spectral}, i.e, given a function $f(\bm{x}, t)$, using Fourier series we may rewrite it as 
\begin{align}
    f(\bm{x},t) = \sum_{k=-\infty}^{\infty} \hat{f}_{k}(t) e^{i k \bm{x}}  
\end{align}
where  $\hat{f}_k(t)$ is the Fourier coefficient corresponding to the wavenumber $k$ at time $t$.

It is also worth noting that both proposed architectures do not introduce any additional trainable parameters compared to conventional PINN models, nor they require significantly more floating point operations to evaluate their forward or backward pass. Therefore they can be used as drop-in replacements to conventional fully-connected architectures with no sacrifices to computational efficiency. In section \ref{sec: results}, we will validate the effectiveness of the proposed architectures through a series of systematic numerical experiments.


\begin{figure}
     \centering
     \begin{subfigure}[b]{0.45\textwidth}
         \centering
         \includegraphics[width=\textwidth]{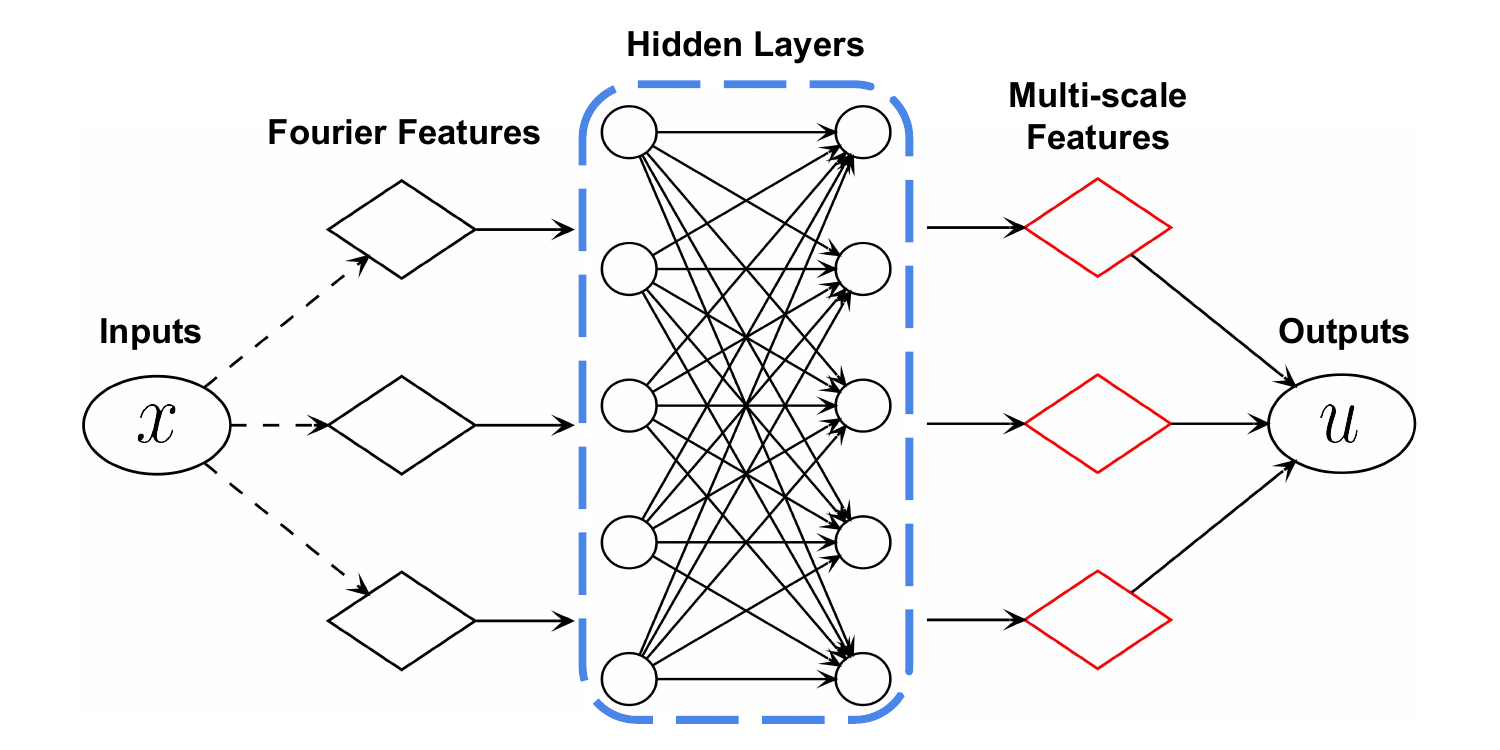}
         \caption{}
         \label{fig: arch_mFF}
     \end{subfigure}
     \begin{subfigure}[b]{0.45\textwidth}
         \centering
         \includegraphics[width=\textwidth]{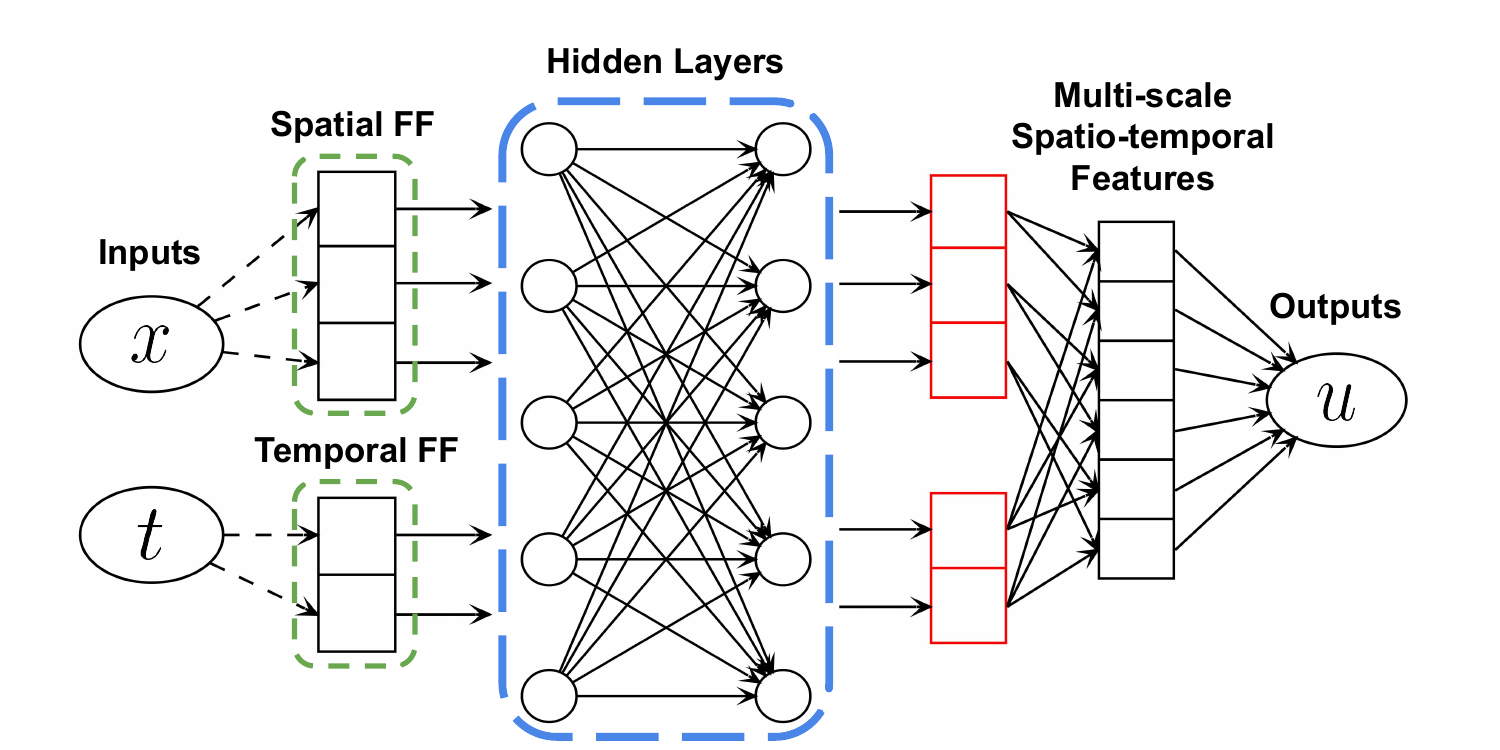}
         \caption{}
         \label{fig: arch_spatail_temporal_mFF}
     \end{subfigure}
        \caption{{\em (a) Multi-scale Fourier feature architecture:} multiple Fourier feature embeddings (initialized with different $\sigma$) are applied to input coordinates and then passed through the same fully-connected neural network, before the outputs are finally concatenated with a linear layer. {\em (b) Spatio-temporal multi-scale Fourier feature architecture:} Multi-scale Fourier feature architecture: multiple Fourier feature embeddings (initialized with different $\sigma$) are separately applied to spatial and temporal input coordinates and then passed through the same fully-connected neural network. Merging of the spatial and temporal outputs is performed using a point-wise multiplication layer, before obtaining the final outputs through a linear layer.}
        \label{fig: architecture}
\end{figure}


\section{Results}
\label{sec: results}


In this section we demonstrate the performance of
the proposed architectures in solving forward and inverse multi-scale problems. Throughout all benchmarks, we employ hyperbolic tangent activation functions and initialize the network using the Glorot normal scheme \cite{glorot2010understanding}. All networks are trained via stochastic gradient descent using the Adam optimizer \cite{kingma2014adam} with defaulting settings. Particularly, we employ  exponential learning rate decay with a decay-rate of 0.9 every $1,000$ training iterations. All results presented in this section can be reproduced using our publicly available codes at \url{https://github.com/PredictiveIntelligenceLab/MultiscalePINNs}.

\subsection{1D Poisson equation}

We begin with a pedagogical example involving the one-dimensional (1D) Poisson equation benchmark described in section \ref{sec: overview_PINNs} and examine the performance of the proposed architectures.

We begin by approximating the latent solution $u(x)$ with the proposed multi-scale Fourier feature architecture (figure \ref{fig: arch_mFF}). Specifically, we apply two Fourier feature mappings to 1D input coordinates, respectively, before passing them through a 2-layer fully-connected neural network with $100$ units per hidden layer, and then concatenating them through a linear layer. Particularly, these Fourier feature mappings are initialized with $\sigma_1=1$ and $\sigma_2 = 10$, respectively. We train the network by minimizing the loss function \ref{eq: Poisson1D_loss} in section \ref{sec: overview_PINNs} under the exactly same hyper-parameter settings.
Figure \ref{eq: Poisson1D} summarizes our results of the predicted solution to the 1D Poisson equation after $40,000$ training iterations. One can see that the predicted solution obtained using the proposed architecture achieves excellent agreement with the exact solution, yielding a $1.36e-03$ prediction error measured in the relative $L^2$-norm. 

Furthermore, we aim to demonstrate that the proposed architecture outperforms the conventional PINNs, as well as the PINNs with a single Fourier feature mapping. To this end, we first train a plain a conventional physics-informed network using exactly the same hyper-parameters and take the resulting relative $L^2$ error as our baseline. Next, we train the same network with a conventional Fourier feature mapping, considering different initializations for $\sigma$ in the range $[1, 50]$, and report the resulting relative $L^2$ errors over 10 independent trials in figure \ref{fig: Poission1D_mFF_diff_sigma}. It can be observed that either plain PINNs or PINNs with vanilla Fourier features fail to attain good prediction accuracy, yielding errors ranging from $10\%$ to above $100\%$  in the relative $L^2$-norm. In particular, we visualize the results obtained using the same network with conventional Fourier features \cite{tancik2020fourier} initialized by $\sigma= 1$ and $\sigma = 50$, respectively. As shown in figure \ref{fig: Poission1D_FF_sigma_1} and figure \ref{fig: Poission1D_FF_sigma_50}, the network with conventional Fourier features initialized by $\sigma=1$ tends to capture the low frequency components of the solution while the one initialized by $\sigma=50$ successfully captures the high frequency oscillations but ignores the low frequency components.

\begin{figure}
    \centering
    \includegraphics[width = 0.9\textwidth]{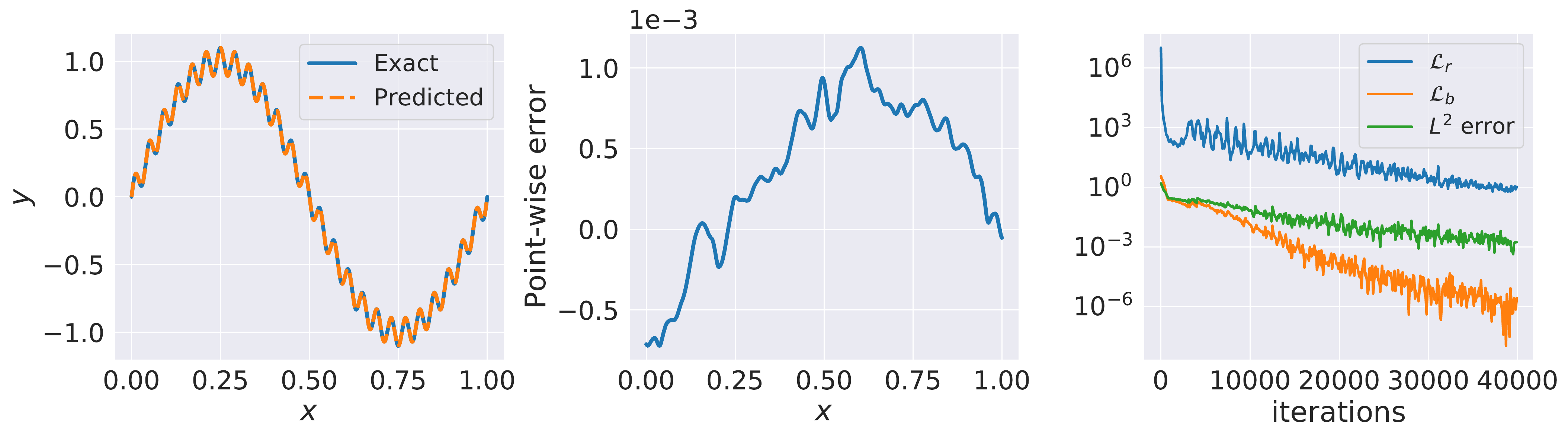}
    \caption{{\em 1D Poisson equation:} Results obtained by training a fully-connected network (2-layer, 100 hidden units, $\tanh$ activations) with the proposed multi-scale Fourier feature mappings via $40,000$ iterations of gradient descent. 
    {\em Left:} Comparison of the predicted and exact solutions. The relative $L^2$ error is $1.36e-03$.
    {\em Middle:} Point-wise error between the predicted and the exact solution. {\em Right:} Evolution of the residual loss $\mathcal{L}_r$, the boundary loss $\mathcal{L}_b$, as well as the relative $L^2$ error  during training. } 
    \label{fig: Poission1D}
\end{figure}

\begin{figure}
    \centering
    \includegraphics[width = 0.4\textwidth]{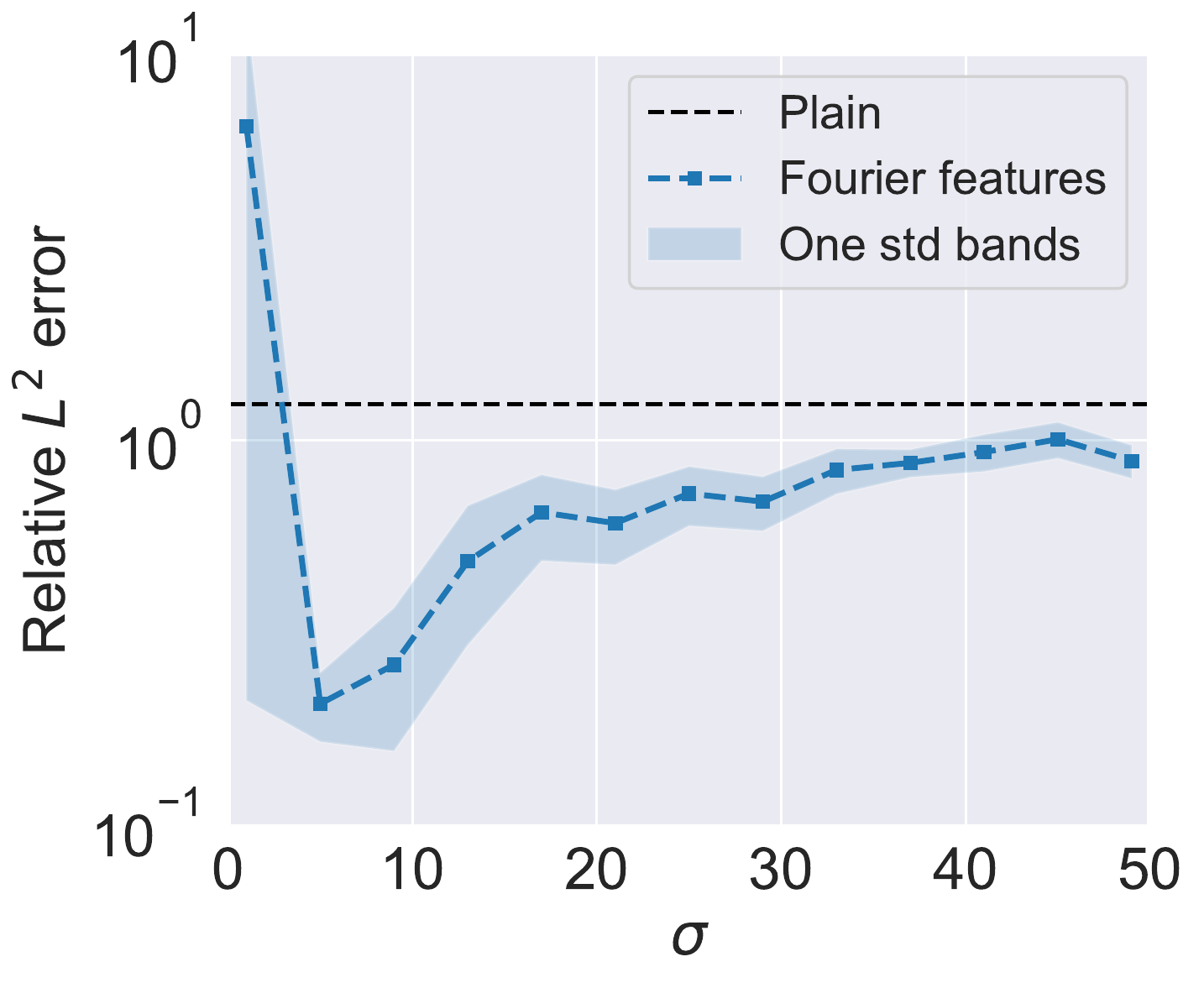}
    \caption{{\em 1D Poisson equation: } Relative $L^2$ errors of predicted solutions averaged over 10 independent trials by training a plain fully-connected neural network (2-layer, 100 hidden units, $\tanh$ activations), as well as the same network with single Fourier feature mapping initialized by different $\sigma \in [1,50]$.}
    \label{fig: Poission1D_mFF_diff_sigma}
\end{figure}

\begin{figure}
     \centering
     \begin{subfigure}[b]{0.9\textwidth}
         \centering
         \includegraphics[width=\textwidth]{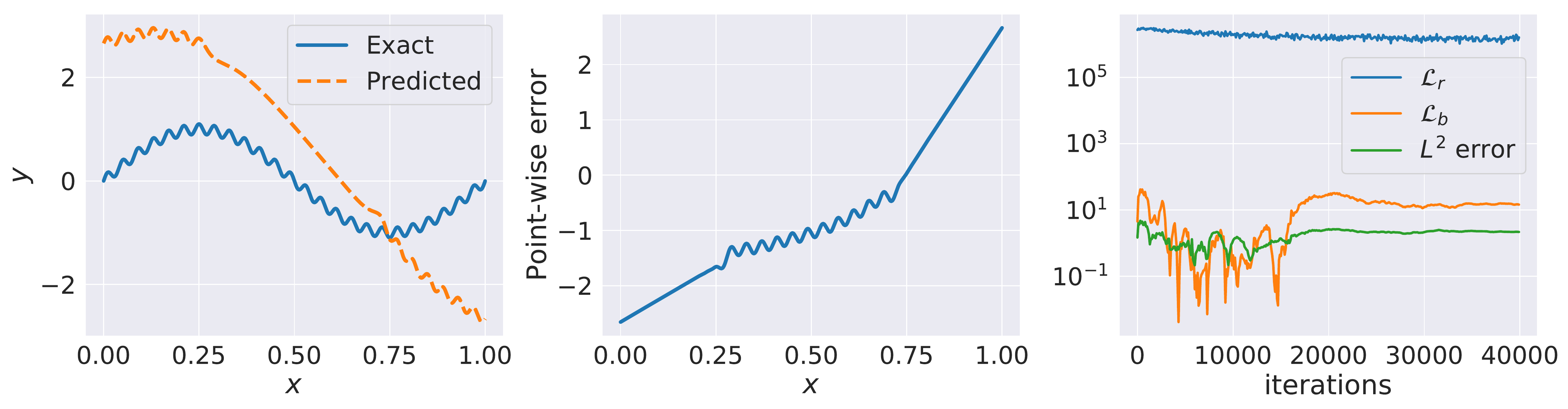}
         \caption{}
         \label{fig: Poission1D_FF_sigma_1}
     \end{subfigure}
     \begin{subfigure}[b]{0.9\textwidth}
         \centering
         \includegraphics[width=\textwidth]{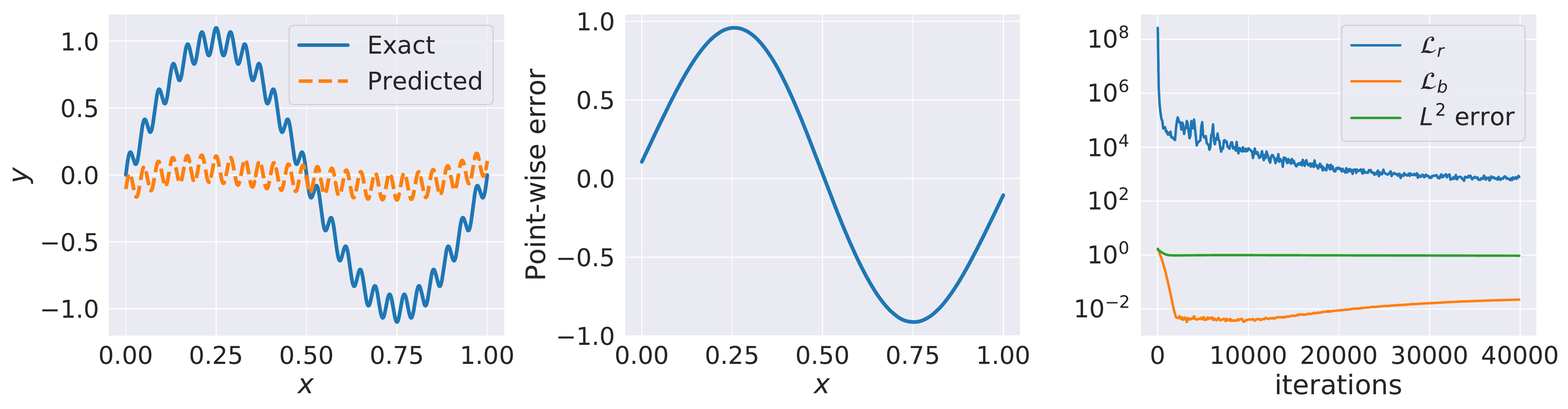}
         \caption{}
         \label{fig: Poission1D_FF_sigma_50}
     \end{subfigure}
         \caption{{\em 1D Poisson equation: } {\em (a)} Results obtained by training a fully-connected network (2-layer, 100 hidden units, $\tanh$ activations) with a single Fourier feature mapping initialized by $\sigma=1$ via $40,000$ iterations of gradient descent.
        {\em Left:} Comparison of the predicted and exact solutions.
        {\em Middle:} Point-wise error between the predicted and the exact solution. {\em Right:} Evolution of the residual loss $\mathcal{L}_r$, the boundary loss $\mathcal{L}_b$,  as well as the relative $L^2$ error  during training.
         {\em (b)} Results obtained by training the same network with single Fourier feature mapping initialized by $\sigma=50$  via 40, 000 iterations of gradient descent.
        {\em Left:} Comparison of the predicted and exact solutions.
        {\em Middle:} Point-wise error between the predicted and the exact solution. {\em Right:} Evolution of the residual loss $\mathcal{L}_r$, the boundary loss $\mathcal{L}_b$,  as well as the relative $L^2$ error  during training.
        } 
        \label{fig: Poission1D_FF_diff_sigma}
\end{figure}

\subsection{High frequencies in a heat equation}

To demonstrate the necessity and effectiveness of the proposed spatio-temporal multi-scale Fourier feature architecture (ST-mFF), let us consider the one-dimensional heat equation taking the form
\begin{align}
    &u_t = \frac{1}{(500 \pi)^2} u_{xx}, \quad (x, t) \in (0,1) \times (0, 1) \\
    &u(x, 0) = \sin(500 \pi x),  \quad x \in [0, 1] \\
    &u(0, t) = u(1, t) = 0,  \quad t \in [0,1].
\end{align}
The exact solution $u(x, t)$ for this benchmark is given by
\begin{align}
    u(x, t) = e^{-t} \sin(500 \pi x).
\end{align}
As the solution is mainly dominated by a single high frequency in the spatial domain, it suffices to employ just single Fourier feature mapping in the network architecture. We proceed by approximating the latent variable $u(x,t)$ with the network $u_{\bm{\theta}}(x)$ using the proposed spatio-temporal architecture (figure \ref{fig: arch_spatail_temporal_mFF}). Specifically, we embed the spatial and temporal input coordinates through two separate Fourier features initialized with $\sigma=200$ and $\sigma = 1$, respectively, and pass the embedded inputs though a 3-layer fully-connected neural network with 100 neurons per hidden layer and finally merge them according to equation \ref{eq: ST_mFF_forward_pass_1} -\ref{eq: ST_mFF_concatenate_2}. The network is trained by minimizing the following loss
\begin{align}
    \mathcal{L}(\bm{\theta})  &=  \mathcal{L}_{bc}(\bm{\theta})  + \mathcal{L}_{ic}(\bm{\theta}) + \mathcal{L}_r(\bm{\theta}) \\
    &= \frac{1}{N_{bc}} \sum_{i=1}^{N_{bc}} \left|u_{\bm{\theta}}(x_{bc}^i, t_{bc}^i) - u(x_{bc}^i, t_{bc}^i)  \right|^2 + \frac{1}{N_{ic}} \sum_{i=1}^{N_{ic}} \left|u_{\bm{\theta}}(x_{ic}^i, t_{ic}^i) - u(x_{ic}^i, t_{ic}^i)  \right|^2 \\
    &+ \frac{1}{N_r}\sum_{i=1}^{N_r} \left| \frac{\partial u_{\bm{\theta}}}{\partial t}(x_r^i, t_r^i) -  \frac{\partial^2 u_{\bm{\theta}}}{\partial x^2}(x_r^i, t_r^i)   \right|^2,
\end{align}
where we set the batch sizes to  $N_{bc} = N_{ic} = N_{r} = 128$.  The predictions of the trained model against the exact solution along with the point-wise absolute error between them are presented in figure \ref{fig: Heat1D_sigma_200_pred}. This figure indicates that our spatio-temporal multi-scale Fourier feature architecture is able to accurately capture the high frequency oscillations, leading to a $1.78e-03$
prediction error measured in the relative $L^2$-norm. 
To the best of author's knowledge, this is the first time that PINNs can be effective in solving a time-dependent problem exhibiting such extremely high frequencies.

Next, we test the performance of the multi-scale Fourier feature architecture. To this end, we represent the unknown solution $u(x,t)$ by the same network (3-layer, 100 hidden units) with single Fourier feature mapping initialized using different $\sigma \in [1, 1000]$. A visual assessment of the resulting relative $L^2$ error, as well as the baseline obtained by the conventional PINNs are shown in figure \ref{fig: Heat1D_FF_l2_error}. As suggested in this figure, all reported relative $L^2$ errors are around $100\%$, which implies that both conventional PINNs and the multi-scale Fourier feature architecture are incapable of learning high frequency components of solutions in the spatio-temporal domain.

\begin{figure}
     \centering
     \begin{subfigure}[b]{0.9\textwidth}
         \centering
         \includegraphics[width=\textwidth]{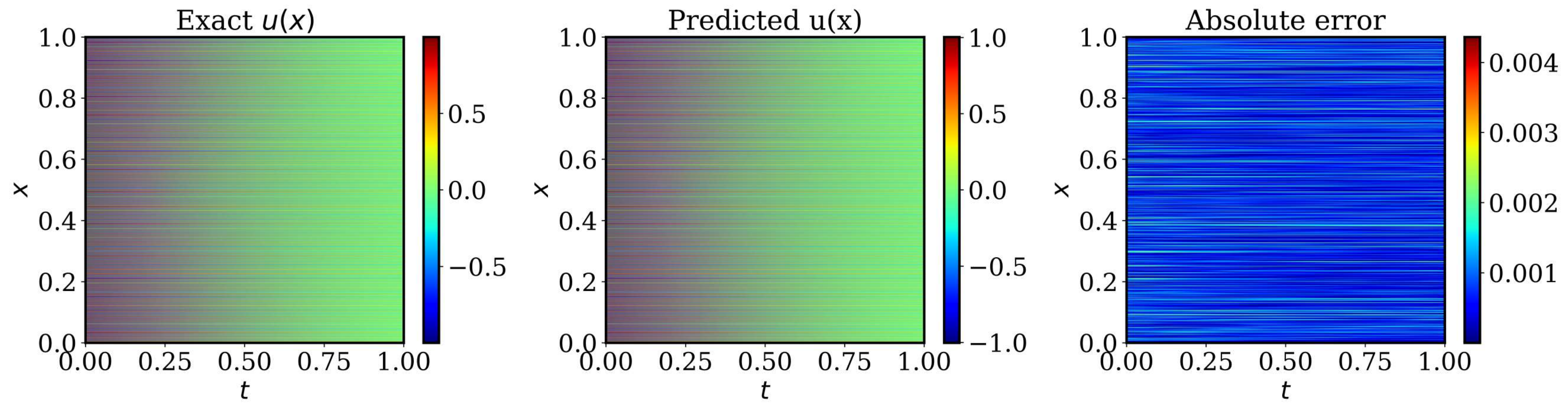}
         \caption{}
         \label{fig: Heat1D_sigma_200_pred}
     \end{subfigure}
     \begin{subfigure}[b]{0.4\textwidth}
         \centering
         \includegraphics[width=\textwidth]{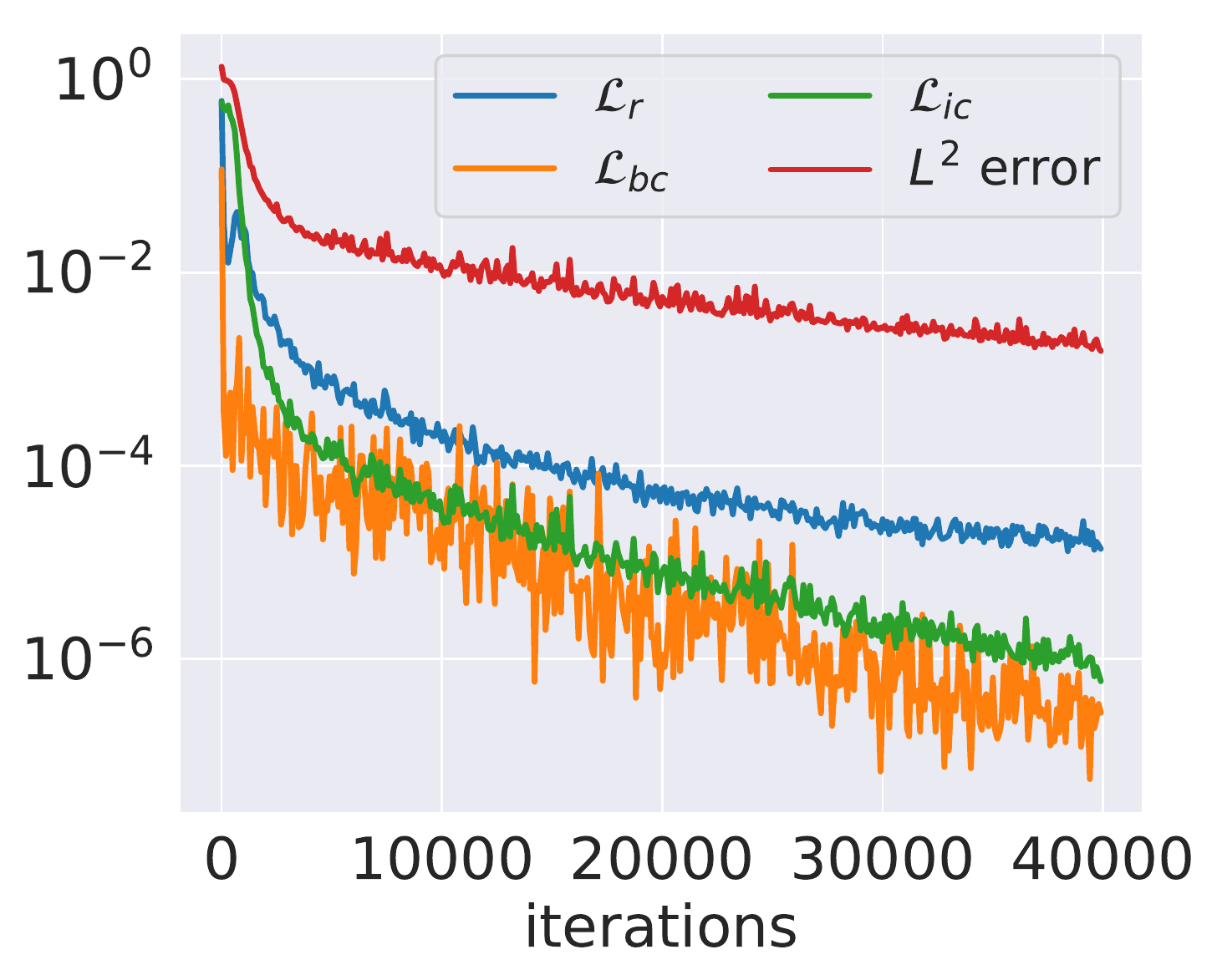}
         \caption{}
         \label{fig: Heat1D_sigma_200_loss}
     \end{subfigure}
        \caption{{\em 1D heat equation:}  {\em (a): } Exact solution versus the predicted solution by training a fully-connected network (3-layer, 100 hidden units, $\tanh$ activations) with the spatio-temporal Fourier feature mappings via 40,000 iterations of gradient descent. The relative $L^2$ error is $1.78e-03$. {\em (b):} Evolution of the different terms in the loss function,  as well as the relative $L^2$ error during training. } 
        \label{fig: Heat1D}
\end{figure}

\begin{figure}
    \centering
    \includegraphics[width=0.4\textwidth]{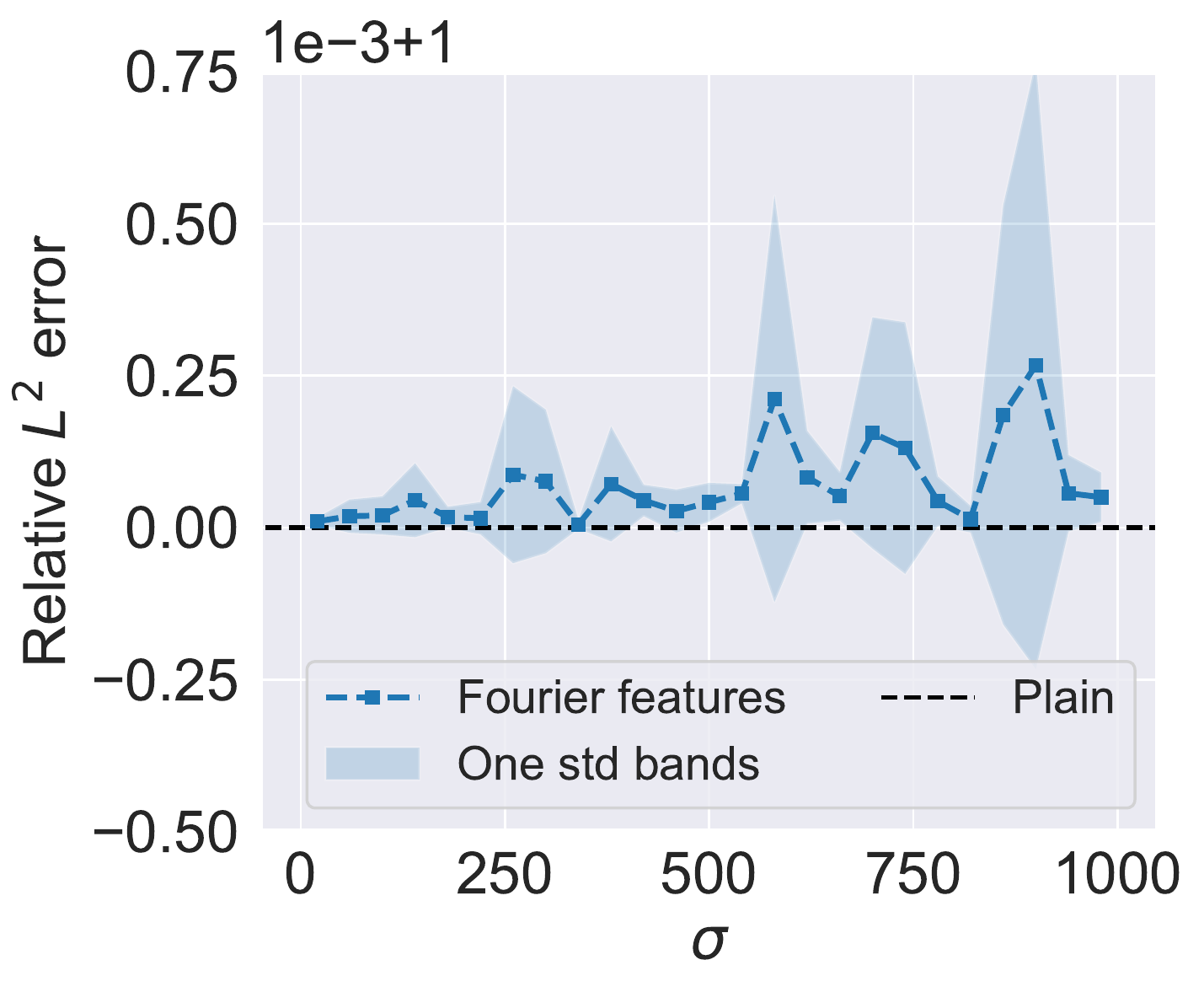}
     \caption{{\em 1D heat equation: } Relative $L^2$ errors of predicted solutions averaged over 10 independent trials by training a plain fully-connected neural network (3-layer, 100 hidden units, $\tanh$ activations), as well as the same network with single Fourier feature mapping initialized by different $\sigma \in [1,1000]$.}
    \label{fig: Heat1D_FF_l2_error}
\end{figure}

\subsection{Wave propagation}

In this example, we aim to demonstrate that employing appropriate architectures solely cannot guarantee accurate predictions. To this end, we consider one-dimensional wave equation taking the form
\begin{align}
    \label{eq: wave_pde}
    &u_{tt}(x, t) - 100 u_{xx}(x, t) = 0, \quad (x, t) \in (0,1) \times (0,1) \\
      \label{eq: wave_bc}
    & u(0, t) = u(1, t) = 0, \quad t \in [0,1] \\
    \label{eq:   wave_u_ic}
    & u(x,0) = \sin( \pi x) +  \sin(2 \pi x), \quad x\in [0,1 ]\\
    \label{eq:   wave_ut_ic}
    & u_t(x,0) = 0, \quad x\in [0,1 ].
\end{align}
By d'Alembert's formula \cite{evans1998partial}, the solution $u(x,t)$ is given by
\begin{align}
    u(x, t) =  \sin( \pi x) \cos(10 \pi t) +  \sin(2 \pi x) \cos(20\pi t).
\end{align}
To handle the multi-scale behavior in both spatial and temporal directions, we employ the spatio-temporal architecture (figure \ref{fig: arch_spatail_temporal_mFF}) to approximate the latent solution $u(x,t)$. Specifically, we apply two separate Fourier feature mappings initialized by $\sigma=1, 10$ respectively to temporal coordinates $t$ 
and apply one Fourier feature  mapping initialized by $\sigma=1$ to spatial coordinates $\bm{x}$. Then we pass all featurized spatial and temporal inputs coordinated through a 3-layer fully-connected neural network with 200 units per hidden layer and concatenate network outputs using equations \ref{eq: ST_mFF_concatenate_1} - \ref{eq: ST_mFF_concatenate_2}. In particular, we treat the initial condition \ref{eq: wave_u_ic} as a special boundary condition on the spatio-temporal domain $\Omega$. Then equation \ref{eq: wave_bc} and equation \ref{eq:   wave_u_ic} can be summarized as
\begin{align*}
    u(\bm{x}) = g(\bm{x}), \quad x \in \partial \Omega
\end{align*}
Then, the network can be trained by minimizing the following loss function
\begin{align}
    \label{eq: loss_wave}
    \mathcal{L}(\bm{\theta}) &=   \mathcal{L}_u(\bm{\theta}) +  \mathcal{L}_{u_t}(\bm{\theta}) +   \mathcal{L}_r(\bm{\theta}) \\
                             &=  \frac1{N_u} \sum_{i=1}^{N_u} |u(\bm{x}_u^i,\bm{\theta}) - g(\bm{x}_u^i)  |^2
                             +  \frac{1}{N_{u_t}} \sum_{i=1}^{N_{u_t}} |u_t(\bm{x}_{u_t}^i,\bm{\theta})   |^2
                             +  \frac{1}{ N_r} \sum_{i=1}^{N_r} |  \frac{\partial^2 u_{\bm{\theta}}}{\partial t^2}(\bm{x}_r^i) -100  \frac{\partial^2 u_{\bm{\theta}}}{\partial x^2}(\bm{x}_r^i)|^2,
 \end{align}
where the batch sizes are set to $N_u = N_{u_t} = N_r = 360$ and all data points $\{\bm{x}_u^i, g(\bm{x}_u^i)   \}_{i=1}^{N_u}, \{\bm{x}_{u_t}^i   \}_{i=1}^{N_{u_t}}$ and $\{\bm{x}_r^i \}_{i=1}^{N_r}$ are uniformly sampled from the appropriate regions in the computational domain at each iteration of gradient descent.    

Figure \ref{fig: Wave1D_STmFF} presents a comparison of the exact and predicted solution obtained after 40,000 iterations of gradient descent. It is evident that the PINN model completely fails to learn the correct solution. This illustrated the fact that even if we choose an appropriate network architecture to approximate the latent PDE solution, there might be some other issues that lead PINNs to fail.
Wang {\em et al.} \cite{wang2020and} found that multi-scale problems are more likely to cause a large discrepancy in the convergence rate of different terms contributing to the total training error, which may lead to severe training difficulties of PINNs in practice. 
Such a  discrepancy can be further justified in figure \ref{fig: Wave1D_loss}, from which one can see that the  loss $\mathcal{L}_r$ and loss $\mathcal{L}_{u_t}$ decrease much faster than the loss $\mathcal{L}_u$ during training. We believe that this may be a fundamental reason behind the collapse of  physics-informed neural networks, and their inability to yield accurate predictions for this specific example. 

To address this a training pathology, we employ the adaptive weights algorithm proposed by Wang {\em et al.} \cite{wang2020and} to train the same network with the same spatio-temporal Fourier feature mappings under the same exactly hyper-parameter settings. As shown in figure \ref{fig: Wave1D_adaptive}, the results demonstrate excellent agreement between the predicted and the exact solution with relative $L^2$ error within $0.1\%$. Furthermore, we test the performance of standard fully-connected network and the proposed multi-scale Fourier feature architecture, with or without the adaptive weights algorithm and the resulting relative $L^2$ errors are summarized in table \ref{tab: Wave1D}. For all results shown in the table, we employ a 3-layer fully-connected network with 200 units per hidden layer as a backbone. To obtain the result of the multi-scale Fourier features architecture, we jointly embed the spatio-temporal coordinates $(x, t)$ using two separate Fourier feature mappings initialized with $\sigma=1, 10$, respectively. We observe that only the proposed spatio-temporal Fourier feature architecture trained with the adaptive weights algorithm of Wang {\em et al.} \cite{wang2020and} is capable of achieving a accurate approximation of the true solution, which highly suggests the necessity of combining the proposed architecture with an appropriate optimization scheme for training.





\begin{figure}
     \centering
     \begin{subfigure}[b]{0.9\textwidth}
         \centering
         \includegraphics[width=\textwidth]{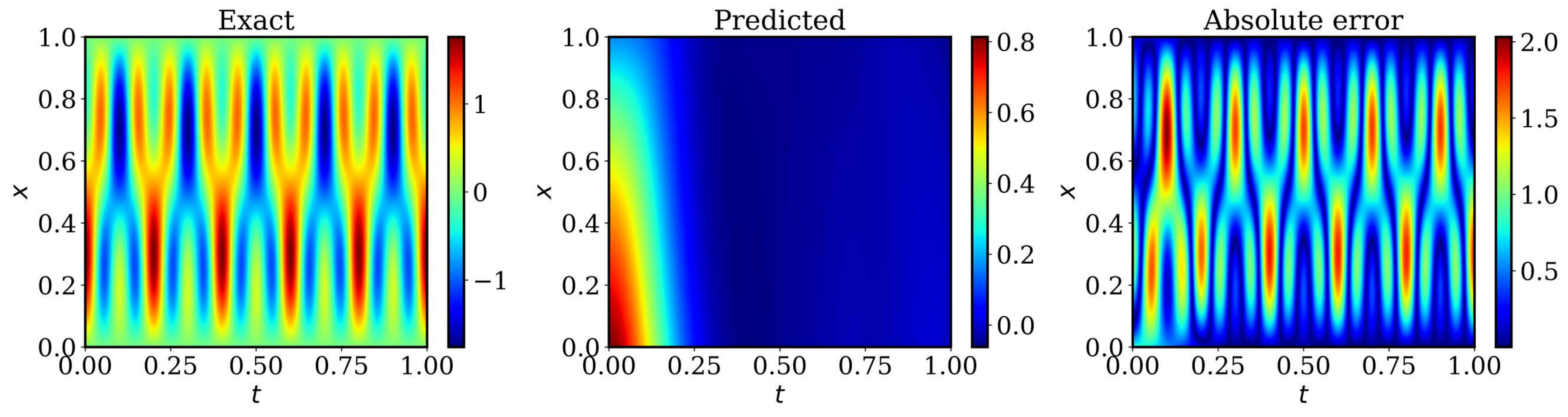}
         \caption{}
         \label{fig: Wave1D_pred}
     \end{subfigure}
     \begin{subfigure}[b]{0.4\textwidth}
         \centering
         \includegraphics[width=\textwidth]{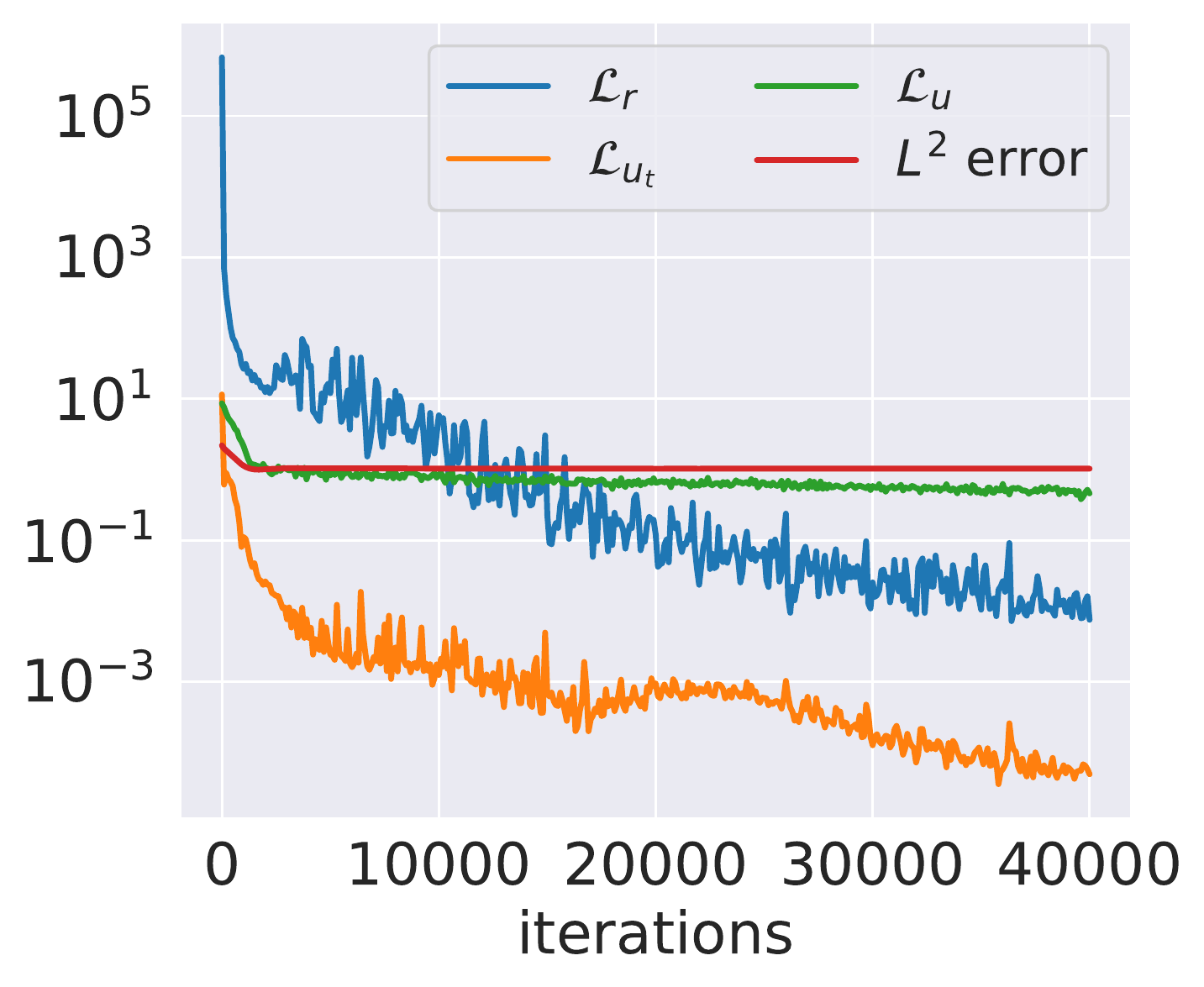}
         \caption{}
         \label{fig: Wave1D_loss}
     \end{subfigure}
         \caption{{\em 1D Wave equation:}  {\em (a):} Exact solution versus the predicted solution by training a fully-connected network (3-layer, 200 hidden units, $\tanh$ activations) with the spatio-temporal Fourier feature mappings 
        via 40,000 iterations of gradient descent. The relative $L^2$ error is $1.03e+00$. {\em (b):} Evolution of the different terms in the loss function,  as well as the relative $L^2$ error during training.} 
        \label{fig: Wave1D_STmFF}
\end{figure}

\begin{figure}
     \centering
     \begin{subfigure}[b]{0.9\textwidth}
         \centering
         \includegraphics[width=\textwidth]{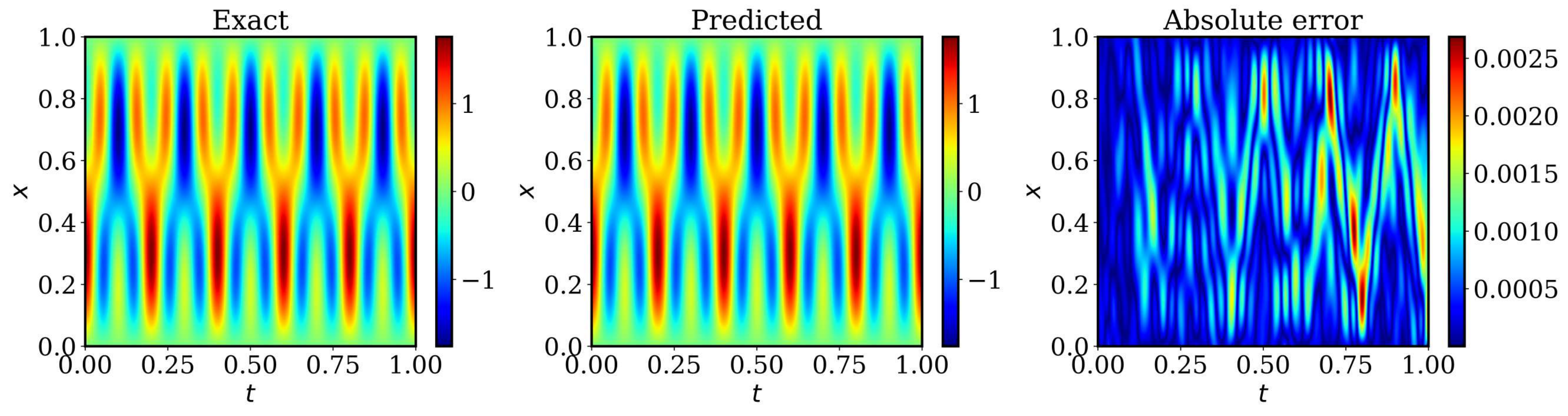}
         \caption{}
         \label{fig: Wave1D_pred_adaptive}
     \end{subfigure}
     \begin{subfigure}[b]{0.4\textwidth}
         \centering
         \includegraphics[width=\textwidth]{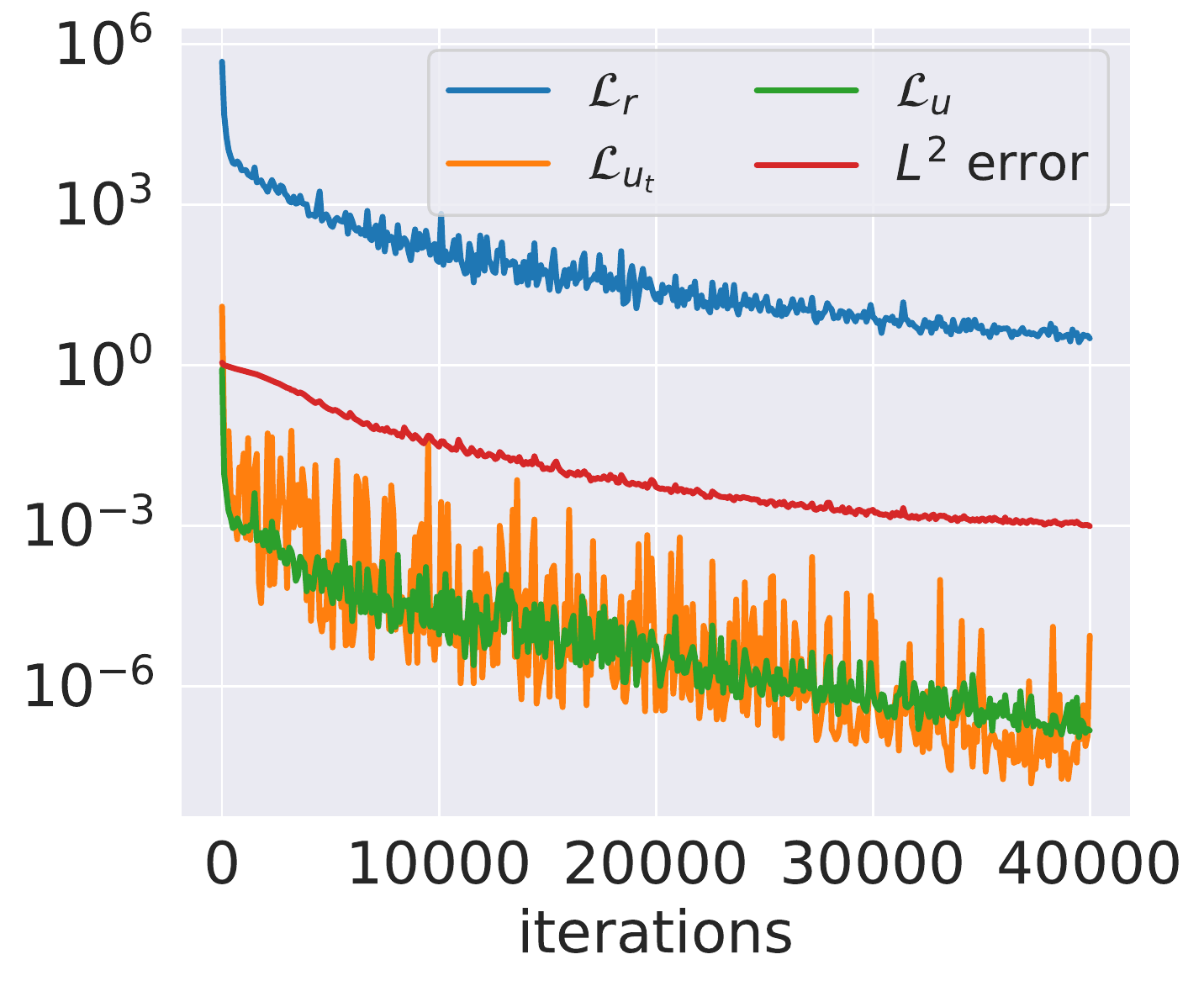}
         \caption{}
         \label{fig: Wave1D_loss_adaptive}
     \end{subfigure}
         \caption{{\em 1D wave equation:}  {\em (a):} Exact solution versus the predicted solution by training a fully-connected network (3-layer, 200 hidden units, $\tanh$ activations) with the spatio-temporal Fourier feature mappings 
         using adaptive weights algorithm \cite{wang2020and} via 40,000 iterations of gradient descent. The relative $L^2$ error is $9.83e-04$. {\em (b):} Evolution of the different terms in the loss function,  as well as the relative $L^2$ error during training. } 
        \label{fig: Wave1D_adaptive}
\end{figure}

\begin{table}[h]
\renewcommand{\arraystretch}{1.4}
    \centering
    \begin{tabular}{|c|c|c|c|}
    \hline
    \diagbox{Method}{Architecture}   & Plain        & MFF     & ST-MFF \\ \hline
    No adaptive weights   & 1.01e+00  & 1.03e+00 & 1.02e+00       \\ \hline
    With adaptive weights &  8.77e-01  & 1.00e+00  &9.83e-04      \\ \hline
    \end{tabular}
    \caption{{\em 1D wave equation:} Relative $L^2$ errors of the predicted solutions obtained using conventional fully-connected networks (plain), multi-scale Fourier features architecture (MFF) and spatio-temporal multi-scale Fourier features architecture (ST-MFF).}
    \label{tab: Wave1D}
\end{table}

\subsection{Reaction-diffusion dynamics in a two-dimensional Gray-Scott model}

Our final example aims to highlight the ability of the proposed methods 
to handle inverse problems. Let us consider a two-dimensional Gray-Scott model \cite{gray1990chemical} that describes two non-real chemical species $U, V$ reacting and transforming to each other. This model is governed by a coupled system of reaction-diffusion equations taking the form 
\begin{align}
    &u_{t}=\varepsilon_{1} \Delta u+b(1-u)-u v^{2} \\
    &v_{t}=\varepsilon_{2} \Delta v-d v+u v^{2}
\end{align}
where $u, v$ represent the concentrations of $U, V$ respectively and $\epsilon_1, \epsilon_2$ are their corresponding diffusion rates.

We generate a data-set containing a direct numerical solution of the two-dimensional Gray-Scott equations with 40,000 spatial points and 401 temporal snapshots. Specifically,  we take $b=0.04, d=0.1, \epsilon_1=2e-5, \epsilon_2=1e-5$ and, assuming periodic boundary conditions, we start from an initial condition 
\begin{align*}
    &u(x, y, 0) = 1 - \exp(-80((x+0.05)^2+(y+0.02)^2), \quad (x, y) \in [-1, 1] \times [-1, 1]  \\
    &v(x, y,0) = \exp(-80((x-0.05)^2+(y-0.02)^2),  \quad (x, y) \in [-1, 1] \times [-1, 1],
\end{align*}
and integrate the equations up to the final time $t =4000$. Synthetic training data for this example are generated using the Chebfun package \cite{driscoll2014chebfun} with a spectral Fourier discretization and a fourth-order stiff time-stepping scheme \cite{cox2002exponential} with time-step size of $0.5$. Temporal snapshots of the solution are are saved every $\Delta t = 10$. 
From this data-set, we create a smaller training subset by collecting the data points from time $t = 3500$ to $t=4000$ (50 snapshots in total) as our training data.  A representative snapshot of the numerical solution is presented in figure \ref{fig: GS2D_ref}. As illustrated in this figure, the solution exhibits complex spatial patterns with a non-trivial frequency content. 

Given training data $\{(x^i, y^i, t^i), (u^i, v^i)\}_{i=1}^{N}$ and assuming that $b,d$ are known, we are interested in predicting the latent concentration fields $u,v$, as well as inferring the unknown diffusion rates $\epsilon_1,\epsilon_2$. To this end, we represent the latent variables $u, v$ by a deep network employing the proposed spatio-temporal Fourier feature architecture
\begin{align}
    (x, y, t) \xrightarrow{\bm{f}_{\bm{\theta}}} (u_{\bm{\theta}}, v_{\bm{\theta}})
\end{align}
To be precise, we map temporal coordinates $t$ by a Fourier feature embedding with $\sigma = 1$ and map spatial coordinates $(x, y)$ by another Fourier feature embedding with $\sigma = 30$. Then we pass the embedded inputs through a 9-layer fully-connected neural network with 100 neurons per hidden layer. 
The corresponding loss function is given by
\begin{align}
    \mathcal{L}(\bm{\theta})  &=  \mathcal{L}_{u}(\bm{\theta})  + \mathcal{L}_{v}(\bm{\theta})  + \mathcal{L}_{r^u}(\bm{\theta}) + \mathcal{L}_{r^v}(\bm{\theta}) \\
    &= \frac{1}{N_u} \sum_{i=1}^{N_u} \left|u_{\bm{\theta}}(x^i, y^i, y^i) - u^i  \right|^2
    + \frac{1}{N_v} \sum_{i=1}^{N_v} \left|v_{\bm{\theta}}(x^i, y^i, y^i) - v^i  \right|^2\\
    &+ \frac{1}{N_{r^u}}\sum_{i=1}^{N_{r^u}} \left| r^u_{\bm{\theta}}(x_r^i, y_r^i, t_r^i)   \right|^2 
    + \frac{1}{N_{r^v}}\sum_{i=1}^{N_{r^v}} \left|r^v_{\bm{\theta}}(x_r^i, y_r^i, t_r^i)    \right|^2,
\end{align}
where the PDE residuals are defined as
\begin{align}
    &r_{\bm{\theta}}^u = \frac{\partial u_{\bm{\theta}} }{\partial t} -  \epsilon_1 \Delta u_{\bm{\theta}} - b(1 - u_{\bm{\theta}})  + u_{\bm{\theta}} v_{\bm{\theta}}^2 \\
  &r_{\bm{\theta}}^v = \frac{\partial v_{\bm{\theta}} }{\partial t} -  \epsilon_2 \Delta v_{\bm{\theta}} + d v_{\bm{\theta}}  - u_{\bm{\theta}} v_
  {\bm{\theta}}^2.
\end{align}
Here we choose batch sizes $N_u = N_v = N_{r^u} = N_{r^v} = 1000$ where all data points along with collocation points are randomly sampled at each iteration of gradient descent. Particularly, since the diffusion rates are strictly positive and generally very small, we parameterize $\epsilon_1, \epsilon_2$ by exponential functions, i.e $\epsilon_i = e^{\alpha_i}$ for $i=1,2$ where $\alpha_i$'s are trainable parameters initialized by $-10$.

We train the network by minimizing the above loss function via via 120,000 iterations of gradient descent.
Figure \ref{fig: GS2D_u_pred_FF} and figure  \ref{fig: GS2D_v_pred_FF} presents the comparisons of reconstructed concentrations $u,v$ against the ground truth functions at the final time $t=4000$. The results show excellent agreement between the predictions and the numerical estimations. This is further validated by the relative $L^2$-norm of error results shown in
figure \ref{fig: GS2D_error_FF}.  Moreover, the evolution of inferred  diffusion rates  during training, as well as the final predictions are presented in figure \ref{fig: GS2D_eps_FF} and table \ref{tab: GS2D_eps_FF} respectively, which show good agreement with the exact values.

To compare these results against the performance of a conventional PINNs model \cite{raissi2019physics}, we also train the same fully-connected neural network (9-layer, width 100) under the same hyper-parameter settings. The results of this experiment are summarized in figure \ref{fig: GS2D_original}. Evidently, conventional PINNs are incapable of accurately learning the concentrations $u,v$, as well as inferring the unknown diffusion rates under the current setting. 

\begin{figure}
    \centering
    \includegraphics[width=0.8\textwidth]{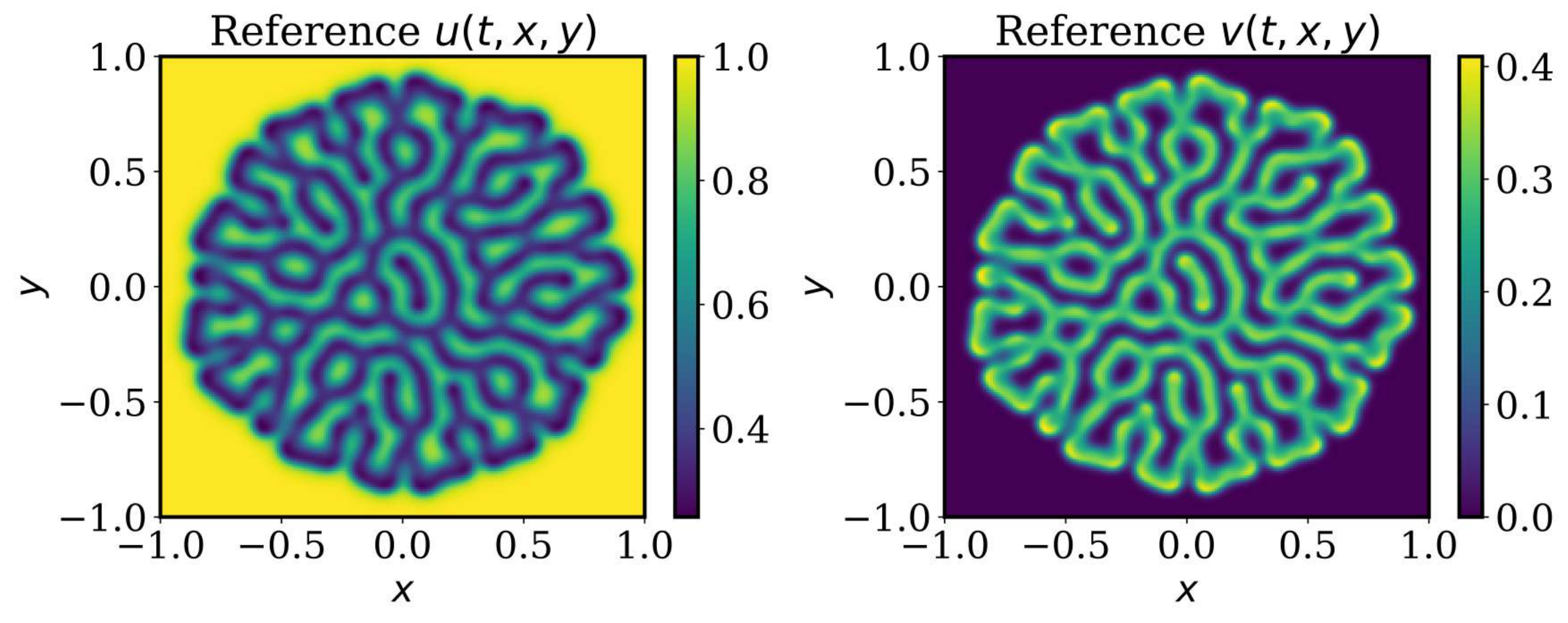}
    \caption{{\em 2D Gray-Scott equation:} Representative snapshots of the ground truth concentration fields $u, v$ at $t=3500$.   }
    \label{fig: GS2D_ref}
\end{figure}

\begin{figure}
     \centering
     \begin{subfigure}[b]{0.9\textwidth}
         \centering
         \includegraphics[width=\textwidth]{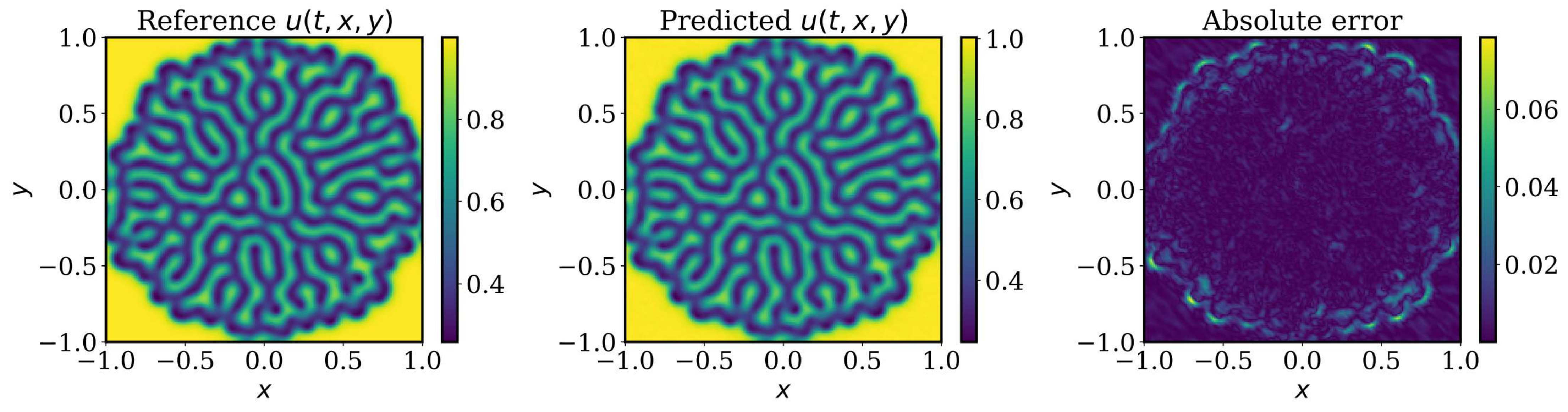}
         \caption{}
         \label{fig: GS2D_u_pred_FF}
     \end{subfigure}
     \begin{subfigure}[b]{0.9\textwidth}
         \centering
         \includegraphics[width=\textwidth]{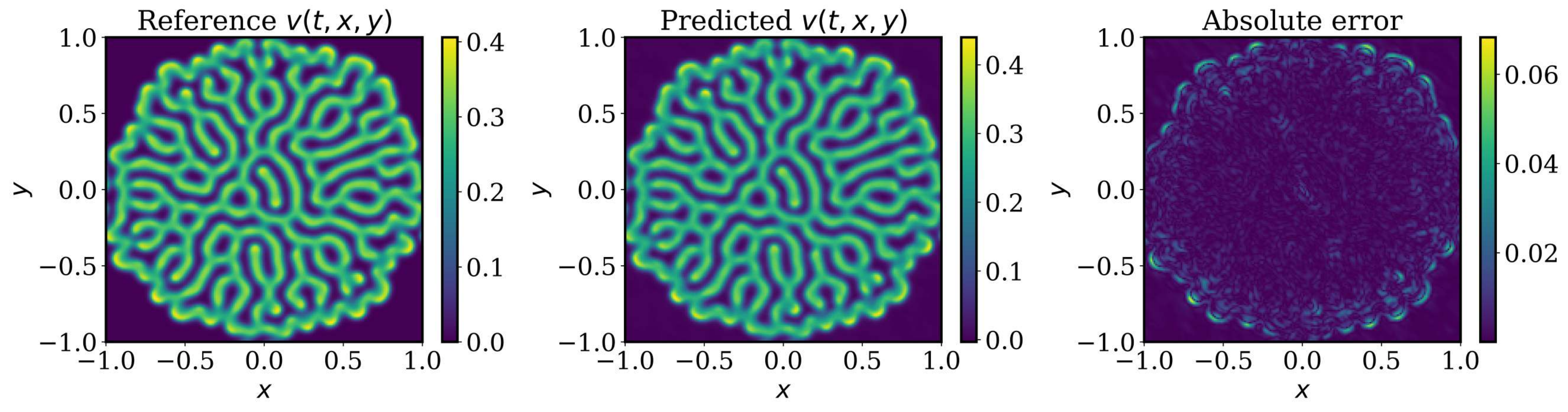}
         \caption{}
         \label{fig: GS2D_v_pred_FF}
     \end{subfigure}
     \begin{subfigure}[b]{0.3\textwidth}
         \centering
         \includegraphics[width=\textwidth]{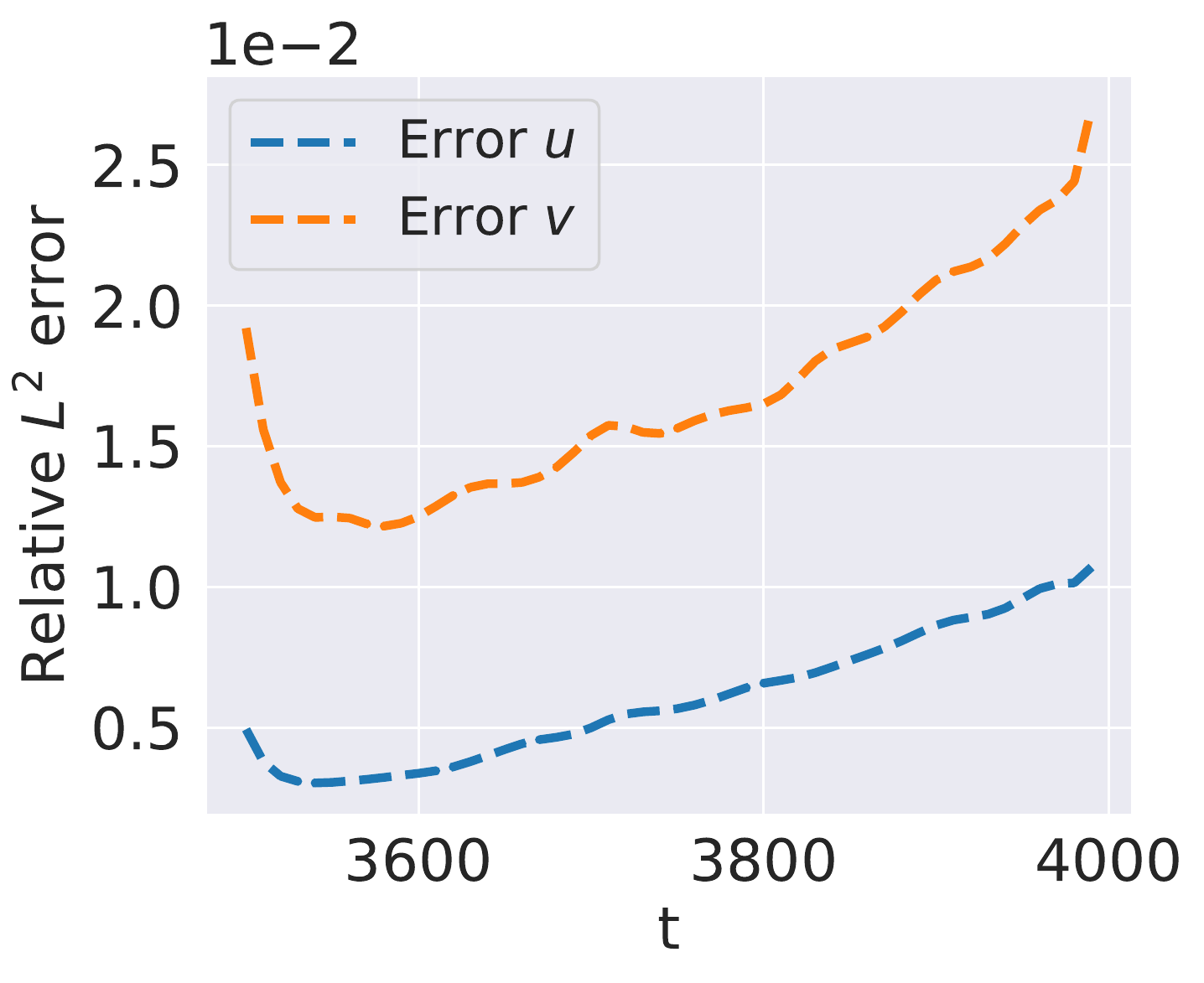}
         \caption{}
         \label{fig: GS2D_error_FF}
     \end{subfigure}
          \begin{subfigure}[b]{0.3\textwidth}
         \centering
         \includegraphics[width=\textwidth]{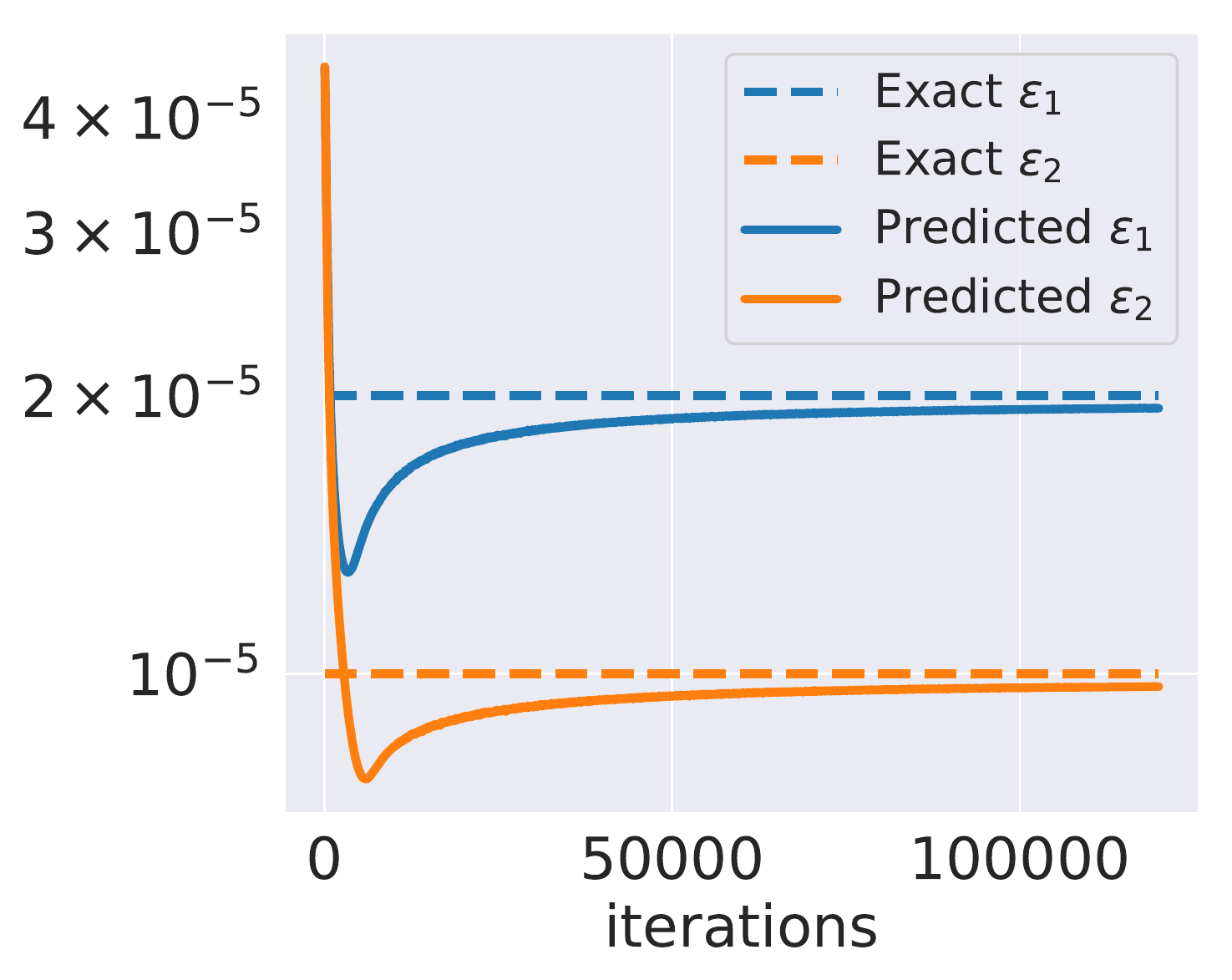}
         \caption{}
         \label{fig: GS2D_eps_FF}
     \end{subfigure}
        \caption{{\em 2D Gray-Scott equation:} Results obtained by training a fully-connected network (9 layers, 100 hidden units, $\tanh$ activations) with the spatio-temporal Fourier feature mappings after 120, 000 iterations of gradient descent.
        (a)(b)  Numerical estimations versus the predicted concentration fields $u, v$ respective at the final time $t =4000$.  (c) Relative $L^2$ errors between the model predictions and the corresponding exact concentration fields for 
        each snapshot $t \in [3500, 4000]$.  (d) Evolution of the inferred diffusion rates $\epsilon_1, \epsilon_2$ during training.}
        \label{fig: GS2D_FF}
\end{figure}

\begin{table}[]
\renewcommand{\arraystretch}{1.4}
    \centering
\begin{tabular}{|c|c|c|c|}
\hline
  Parameters        & Exact & Learned  & Relative $L^2$ error \\ \hline
$\epsilon_1$ & $2e-05$ & $ 1.95e-05$ & $3.00\% $                             \\ \hline
$\epsilon_2$ & $1e-05$ & $ 9.70e-06$ & $ 8.06\% $                           \\ \hline
\end{tabular}
    \caption{{\em  2D Gray-Scott equation:} Exact diffusion rates versus the inferred diffusion rates after training.}
    \label{tab: GS2D_eps_FF}
\end{table}

\begin{figure}
     \centering
     \begin{subfigure}[b]{0.9\textwidth}
         \centering
         \includegraphics[width=\textwidth]{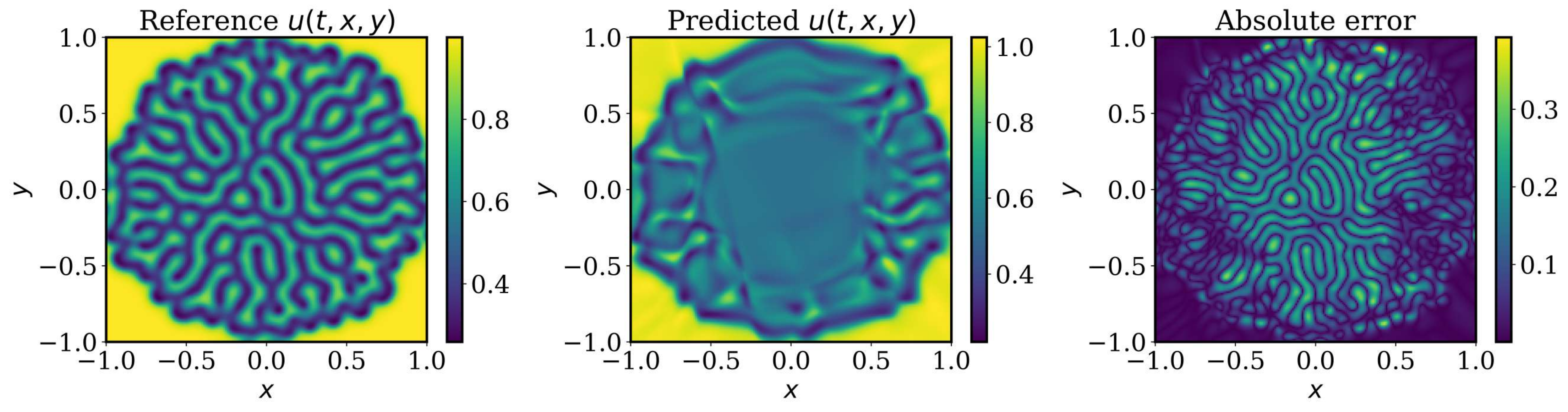}
         \caption{}
         \label{fig: GS2D_u_pred_original}
     \end{subfigure}
     \begin{subfigure}[b]{0.9\textwidth}
         \centering
         \includegraphics[width=\textwidth]{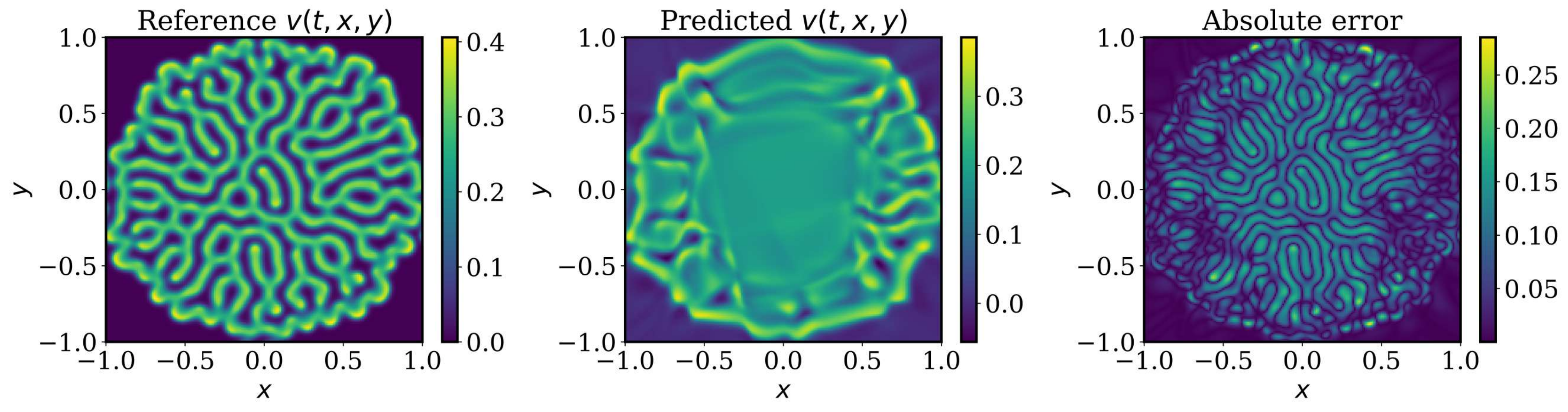}
         \caption{}
         \label{fig: GS2D_v_pred_original}
     \end{subfigure}
     \begin{subfigure}[b]{0.3\textwidth}
         \centering
         \includegraphics[width=\textwidth]{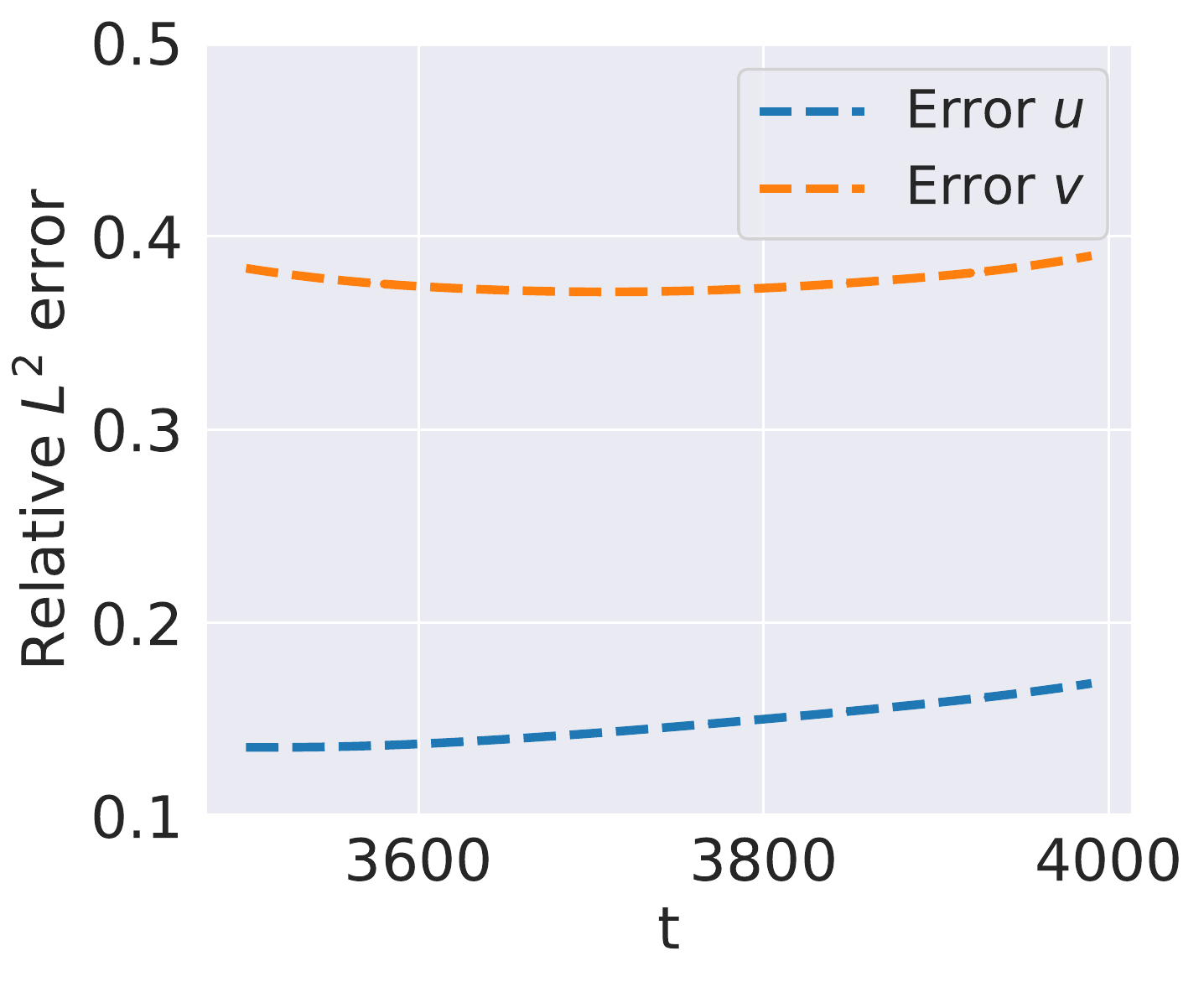}
         \caption{}
         \label{fig: GS2D_error_original}
     \end{subfigure}
          \begin{subfigure}[b]{0.3\textwidth}
         \centering
         \includegraphics[width=\textwidth]{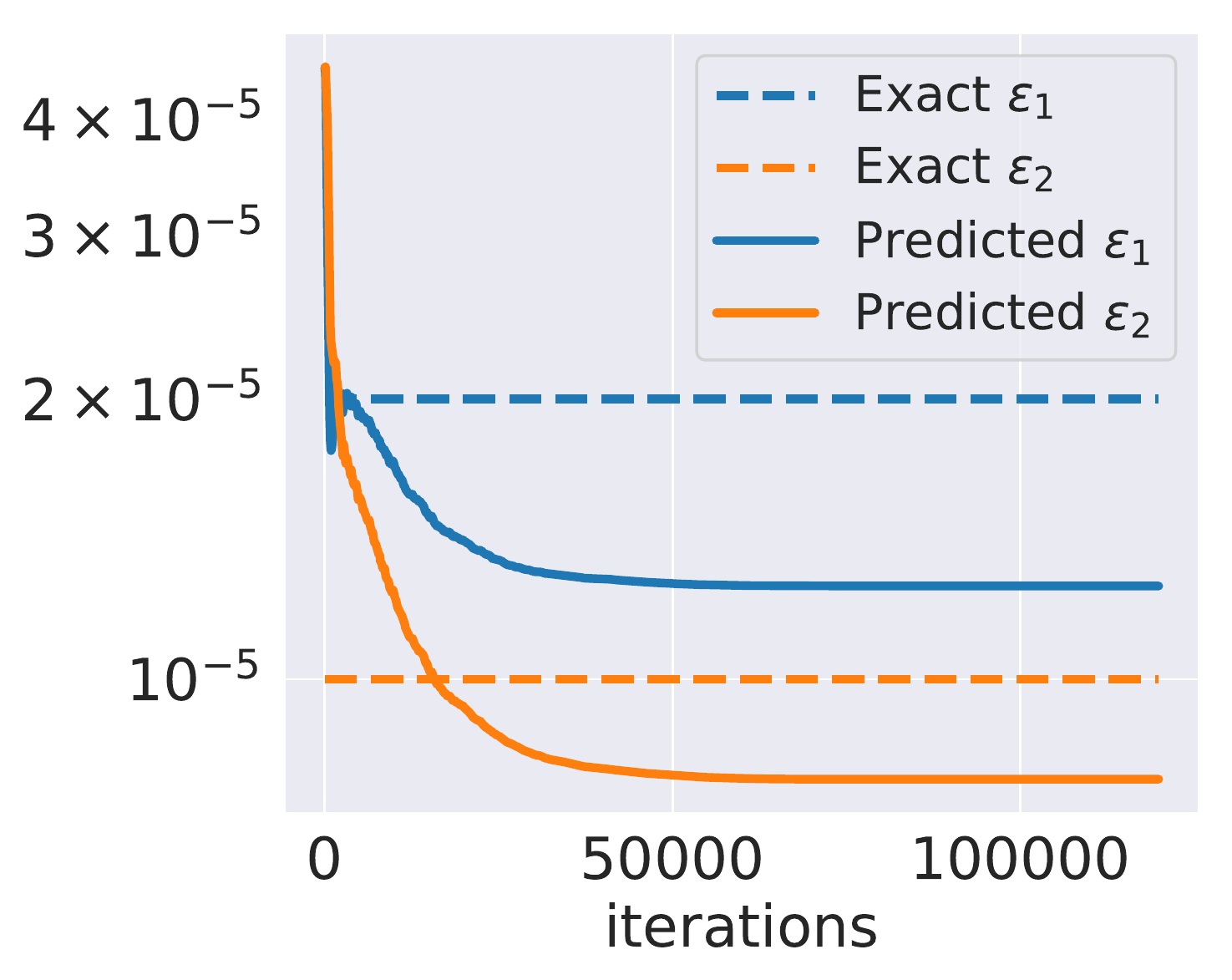}
         \caption{}
         \label{fig: GS2D_eps_original}
     \end{subfigure}
        \caption{{\em 2D Gray-Scott equation:} Results obtained by training a fully-connected network (9 layers, 100 hidden units, $\tanh$ activations) after 120, 000 iterations of gradient descent.
        (a)(b)  Numerical estimations versus the predicted concentration fields $u, v$ respective at the final time  $t=4000$.  (c) Relative $L^2$ errors between the model predictions and the corresponding exact concentration fields for each snapshot $t \in [3500, 4000]$.  (d) Evolution of the inferred diffusion rates $\epsilon_1, \epsilon_2$ during training.}
        \label{fig: GS2D_original}
\end{figure}

\section{Discussion}
\label{sec: disscusion}

In this work, we study Fourier feature networks through the lens of their limiting neural tangent kernel, and show that Fourier feature mappings determine the frequency of the eigenvectors of the resulting NTK. This analysis sheds light into mechanisms that introduce spectral bias in the training of deep neural networks, and suggests possible avenues for overcoming this fundamental limitation. Specific to the context of physics-informed neural networks, our analysis motivates the design of two novel network architectures to tackle forward and inverse problems involving time-dependent PDEs with solutions that exhibit complex multi-scale spatio-temporal features. To gain further insight, we propose a series of benchmarks for which conventional PINN approaches fail, and demonstrate the effectiveness of the proposed methods under these challenging settings.
Taken together, the developments presented in this work provide a principled way of analyzing the performance of PINN models, and enable the design of a new generation of architectures and training algorithms that introduce significant improvements both in terms of training speed and generalization accuracy, especially for multi-scale PDEs for which current PINN models struggle.

Despite this progress, we must admit that we are still at the very early stages of tackling realistic multi-scale and multi-physics problems with PINNs. 
One main limitation of the
the proposed architectures is that we have to carefully choose the appropriate number of Fourier feature mappings and their scale, such that the frequency of the NTK eigenvectors and the target function are roughly matched to each other. In other words, the proposed architectures require some prior knowledge regarding the frequency distribution of the target PDE solution. However, this kind of information may not be accessible for some forward problems, especially for more complex dynamical systems involving the fast transitions of frequencies such the Kuramoto-Sivashinsky equation \cite{sivashinsky1977nonlinear, kuramoto1978diffusion}, or the Navier-Stokes equations in the turbulent regime.  
Fortunately, this issue could be mitigated for some inverse problems where we may perform some spectral analysis on the training data to determine the appropriate number and scale of the Fourier feature mappings.

There are also many open questions worth considering as future research directions. From a theoretical standpoint, we numerically verify that the frequency of Fourier feature mappings determines the frequency of the NTK eigenvectors. Can we rigorously establish a theory for general networks? Besides, what is the behavior of the NTK eigensystem of PINNs? What is the difference between the resulting eigensystem of PINNs and conventional neural networks?
From a practical standpoint, we observe that the eigenvalue distribution moves outward (see figure \ref{fig: reg_a_20_sigma_1_spec_error}) when choosing an inappropriate scale of Fourier feature mappings, which implies that the parameters of the network have to move far away from their initialization to find a good local minimum. Thus, it is natural to ask how to initialize PINNs such that the desirable local minima exist in the vicinity of the parameter initialization in the corresponding loss landscape? Moreover, can we design other useful feature embeddings that aim to handle different scenarios (e.g., shocks, boundary layers, etc.)? We believe that answering these questions not only paves a new way to better understand PINNs and their training dynamics, but also opens a new door for developing scientific machine learning algorithms with provable convergence guarantees, as needed for many critical applications in computational science and engineering.

\section*{Acknowledgements}
PP acknowledges support from the DARPA PAI program (grant HR00111890034), the US Department of Energy (grant DE-SC0019116), and the Air Force Office of Scientific Research (grant FA9550-20-1-0060).

\bibliographystyle{unsrt}
\bibliography{references}


\appendix
\input{appendix}

\end{document}

%% file: appendix.tex
\section{Definition of fully-connected neural networks}
\label{appendix: def_FCNN}
A scalar-valued fully-connected neural network with $L$ hidden layers is defined recursively as follows
\begin{align}
    & \bm{f}^{(0)}(\bm{x}) = \bm{x} \\
    &\bm{f}^{(h)}(\bm{x}) =  \frac{1}{\sqrt{d_h}} \bm{W}^{(h)} \cdot \bm{g}^{(h)}  + \bm{b}^{(h)}  \in \R^{d_{h+1}}, \\
    &\bm{g}^{(h)}(\bm{x}) =
    \sigma (\bm{W}^{(h-1)}  \bm{f}^{(h-1)}(\bm{x})  + \bm{b}^{(h-1)}),
\end{align}
for $h = 1, \dots, L$, where $\bm{W}^{(h)} \in \R^{d_{h+1} \times d_{h}}$ are  weight matrices and $\bm{b}^{(h)} \in \R^{d_{h+1}}$ are bias vectors in the $h$-th hidden layer, and $\sigma: \R \rightarrow \R$ is a coordinate-wise smooth activation function. The final output of the neural network is given by
\begin{align}
    \label{eq: NTK_param}
    f(\bm{x}, \bm{\theta}) &= \bm{f}^{(L)}(\bm{x}) =  \frac{1}{\sqrt{d_L}}\bm{W}^{(L)} \cdot \bm{g}^{(L)}(\bm{x}) + \bm{b}^{(L)},
\end{align}
where $\bm{W}^{(L)} \in \R^{1 \times d_{L}}$ and $\bm{b}^{(L)} \in \R$ are the weight and bias parameters of the last layer. 
Here, $\bm{\theta} = \{\bm{W}^{(0)}, \bm{b}^{(0)}, \dots, \bm{W}^{(L)}, \bm{b}^{(L)}\}$ denotes all parameters of the network and are initialized as independent and identically distributed (i.i.d) Gaussian random variables  $\mathcal{N}(0,1)$. We remark that such a parameterization is known as the ``NTK parameterization'' following the original work of Jacot {\em et. al.} \cite{jacot2018neural}.

\section{Proof of Lemma \ref{lemma: eigenfunc}}
\label{proof: lemma}

\begin{proof}
Recall that $K(\bm{x},\bm{x}') =  \frac{1}{m} \sum_{k=1}^m \cos(\bm{b}_k^T (\bm{x} - \bm{x}'))$, where $\bm{x} = (\bm{x}_1, \dots, \bm{x}_l, \dots, \bm{x}_d)$. By equation \ref{eq: eigenfunc_equ}, we have 
\begin{align}
        \frac{1}{m} \sum_{k=1}^m  \int_C \cos(\bm{b}_k^T ( \bm{x} - \bm{x}')) g\left(\bm{x}'\right) d \bm{x}'=\lambda g\left(\bm{x}\right).
\end{align}
For $l=1,2,\dots, d$, taking derivatives with respect to $\bm{x}_l$ gives
\begin{align}
     -\frac{1}{m} \sum_{k=1}^m \bm{b}_{kl} \int_C \sin(\bm{b}_k^T (\bm{x} - \bm{x}')) g\left(\bm{x}'\right) d \bm{x}'=\lambda \frac{\partial g\left(\bm{x}\right)}{\partial \bm{x}_l}.
\end{align}
Again taking derivatives with respect to $\bm{x}_l$ yields 
\begin{align}
      - \frac{1}{m} \sum_{k=1}^m \bm{b}_{kl}^2 \int_C \cos(\bm{b}_k^T (\bm{x} - \bm{x}')) g\left(\bm{x}'\right) d \bm{x}'=\lambda \frac{\partial^2 g\left(\bm{x}\right)}{\partial \bm{x}_l^2}.
\end{align}
Summing over $l$ gives
\begin{align*}
     - \frac{1}{m} \sum_{k=1}^m \sum_{l=1}^d \bm{b}_{kl}^2 \int_C \cos(\bm{b}_k^T (\bm{x} - \bm{x}')) g\left(\bm{x}'\right) d \bm{x}' = \lambda \sum_{l=1}^d \frac{\partial^2 g\left(\bm{x}\right)}{\partial \bm{x}_l^2}.
\end{align*}
Note that  $\|B\|_F^2 =  \sum_{k,l}\bm{b}_{kl}^2$ and $\Delta = \sum_l \partial_{\bm{x}_l}^2$. Therefore, we have
\begin{align}
    - \frac{1}{m} \|\bm{B}\|_F^2 \lambda g(\bm{x}) = \lambda \Delta g(\bm{x}).
\end{align}
When $\lambda \neq 0$, the above equation is equivalent to
\begin{align}
          \Delta g(\bm{x}) = - \frac{1}{m} \|\bm{B}\|_F^2  g(\bm{x}).
\end{align}
This concludes the proof.

\end{proof}

\section{Proof of Proposition \ref{prop: eigenfunc}}
\label{proof: prop_1}
\begin{proof}
Suppose that $g(x)$ is an eigenfunction of the integral operator $K(x,x') = \cos(b(x - x'))$ with respect to the non-zero eigenvalue $\lambda$, i.e.
\begin{align}
        \label{eq: eigenfunc_1}
        \int_0^1 \cos(b(x - x')) g\left(\bm{x}'\right) d \bm{x}'=\lambda g\left(\bm{x}\right).
\end{align}
By Lemma \ref{lemma: eigenfunc}, first we know that
$g(x)$ satisfies 
\begin{align}
    g''(x) = -b^2 g(x).
\end{align}
Note that this is a ODE and thus $g(x)$ must have the form of 
\begin{align}
       g(x) = C_1 \cos(b x) + C_2 \sin(b x),
\end{align}
where $C_1, C_2$ are some constants. 

Next, we compute the corresponding eigenvalue.  
Substituting the expression of $g$ into equation \ref{eq: eigenfunc_1} we get
\begin{align*}
      \int_0^1 \cos(b (x - x')) \left[C_1 \cos(b  x') + C_2 \sin(b x') \right] d x' 
     =\lambda  \left[C_1 \cos(b  x') + C_2 \sin(b x') \right].
\end{align*}
Then,
\begin{align*}
     LHS &=  \int_0^1 \cos(b (x - x')) \left[C_1 \cos(b  x') + C_2 \sin(b x') \right]  d x' \\
    &=  \int_0^1 \cos(b x)\cos(b x')  + \sin(b x)\sin(b x')  \left[C_1 \cos(b  x') + C_2 \sin(b x') \right]  d x' \\ 
    &= \cos(bx) \int_0^1 \cos(b x') \left[C_1 \cos(b  x') + C_2 \sin(b x') \right]  d x'  +\sin(bx) \int_0^1 \sin(b x') \left[C_1 \cos(b  x') + C_2 \sin(b x') \right]  d x' \\
    &= \cos(bx) [C_1 I_1 + C_2 I_2] + \sin(bx) [C_1 I_2 + C_2 I_3],
\end{align*}
where
\begin{align*}
    &I_1 = \int_0^1 \cos(bx') \cos(bx') dx' = \frac{1}{2} + \frac{\sin(2b)}{4b}, \\
    &I_2 =  \int_0^1 \cos(bx') \sin(bx') dx' = \frac{1 - \cos(2b)}{4b},  \\
    &I_3 = \int_0^1 \sin(bx') \sin(bx') dx' = \frac{1}{2} - \frac{\sin(2b)}{4b}.
\end{align*}
Since $LHS =RHS$, we have
\begin{align}
  [C_1 I_1 + C_2 I_2]  \cos(bx)  + [C_1 I_2 + C_2 I_3] \sin(bx)   = \lambda  \left[C_1 \cos(b  x) + C_2 \sin(b x) \right],
\end{align}
which follows
\begin{align}
    &C_1 I_1 + C_2 I_2 = \lambda C_1 \\
    &C_1 I_2 + C_2 I_3 = \lambda C_2.
\end{align}
Let  $  A =   \begin{bmatrix}
    I_1 & I_2 \\
    I_2 & I_3
    \end{bmatrix}$ and $v = \begin{bmatrix}
    C_1 \\
    C_2
    \end{bmatrix}$. Then, the above linear system can be written as
\begin{align*}
 A v = 
    \lambda v.
\end{align*}
This means that the eigenvalue of the kernel function $K(x,x')$ is determined by the eigenvalue of the matrix $A$. The characteristic polynomial is given by
 \begin{align*}
        \det(\lambda I - A) = (\lambda - I_1)(\lambda - I_3) - I_2^2 = \lambda^2 - (I_1 + I_3) \lambda + I_1 I_3 - I_2^2.
\end{align*}
Note that $I_1 + I_3 = 1 $, and
\begin{align*}
         I_1 I_3 - I_2^2 &= \left(\frac{1}{2} + \frac{\sin(2b)}{4b}  \right) \left(\frac{1}{2} - \frac{\sin(2b)}{4b}  \right) - \left(\frac{1 - \cos(2b)}{4b} \right)^2 \\
    &= \frac{1}{4} - \frac{\sin^2(2b)}{16 b^2} - \frac{1 - 2 \cos(2b) - \cos^2(2b)}{16b^2} \\
    &= \frac{1}{4} - \frac{2 - 2 \cos(2b)}{16 b^2} \\
    & = \frac{1}{4} - \frac{\sin^2(b)}{4 b^2}.
\end{align*}
Hence we have 
\begin{align*}
     \det(\lambda I - A) = \lambda^2 - \lambda +   \frac{1}{4} - \frac{\sin^2(b)}{4 b^2}.
\end{align*}
Then the eigenvalues are
\begin{align*}
    \lambda_1 = \frac{1 + \frac{\sin b}{b}}{2}, \quad  \lambda_2 = \frac{1 - \frac{\sin b}{b}}{2}
\end{align*}
which completes the proof.

\end{proof}

\section{Proof of Proposition \ref{prop: pinns}}

\begin{proof}
By assumption (i) in Proposition \ref{prop: pinns}, we immediately obtain 
\begin{align}
    \begin{bmatrix}
    \frac{d \mathcal{B}[\bm{u}] (\bm{x}_b, {\bm \theta}(t))}{dt}\\
    \frac{d \mathcal{N}[\bm{u}](\bm{x}_r, {\bm \theta}(t))}{dt}
    \end{bmatrix}
    \approx
       - \begin{bmatrix}
     \bm{K}_{uu}(0) & \bm{K}_{ur}(0) \\
     \bm{K}_{ru}(0) & \bm{K}_{rr}(0)
    \end{bmatrix}
    \cdot
       \begin{bmatrix}
    \mathcal{B}[\bm{u}](\bm{x}_b, {\bm \theta}(t)) - \bm{g}(\bm{x}_b) \\
    \mathcal{N}[\bm{u}](\bm{x}_r, {\bm \theta}(t)) - \bm{f}(\bm{x}_r)
    \end{bmatrix},
\end{align}
which implies
\begin{align}
     \begin{bmatrix}
     \mathcal{B}[\bm{u}](\bm{x}_b, {\bm \theta}(t))\\
  \mathcal{N}[\bm{u}](\bm{x}_r, {\bm \theta}(t))
    \end{bmatrix} -
     \begin{bmatrix}
     \bm{g}(\bm{x}_b) \\
    \bm{f}(\bm{x}_r)
    \end{bmatrix}
    &\approx    \left( I - e^{- \bm{K}(0)t} \right)  \cdot
    \begin{bmatrix}
     \bm{g}(\bm{x}_b) \\
    \bm{f}(\bm{x}_r)
    \end{bmatrix}  -
     \begin{bmatrix}
     \bm{g}(\bm{x}_b) \\
    \bm{f}(\bm{x}_r)
    \end{bmatrix} \\
    & \approx -  e^{- \bm{K}(0) t}  \cdot
    \begin{bmatrix}
     \bm{g}(\bm{x}_b) \\
    \bm{f}(\bm{x}_r)
    \end{bmatrix}.
\end{align}
By assumption (ii) in Proposition \ref{prop: pinns}, $\bm{K}_{uu}(0)$ $\bm{K}_{rr}(0)$ are positive definite, and there exist orthogonal matrix $\bm{Q}_u$ and $\bm{Q}_r$ such that
\begin{align}
    &\bm{K}_{uu}(0) =  \bm{Q}_u^T \Lambda_u  \bm{Q}_u^T, \\
    &\bm{K}_{rr}(0) =  \bm{Q}_r^T \Lambda_r  \bm{Q}_r^T, 
\end{align}
where $\Lambda_u$ and $\Lambda_r$ are diagonal matrices whose entries are eigenvalues of $\bm{K}_{uu}(0)$ and $\bm{K}_{rr}(0)$, respectively. We remark that $\Lambda_u$ and $\Lambda_r$ are invertible since all eigenvalues are strictly positive.

Now let $ \bm{Q} = \begin{bmatrix}
             \bm{Q}_u & 0\\
           0 &  \bm{Q}_r
           \end{bmatrix}$  
to obtain
\begin{align}
    \bm{Q}^T \bm{K}(0)  \bm{Q} &=  \begin{bmatrix}
             \bm{Q}_u^T & 0\\
           0 &  \bm{Q}_r^T
           \end{bmatrix}
           \begin{bmatrix}
            \bm{K}_{rr}(0) & \bm{K}_{ur}(0)\\
            \bm{K}_{ur}^T(0) & \bm{K}_{rr}(0)
           \end{bmatrix}
           \begin{bmatrix}
            \bm{Q}_u & 0\\
           0 & \bm{Q}_r
           \end{bmatrix} \\
           & = \begin{bmatrix}
            \bm{\Lambda}_u & \bm{Q}_u^T \bm{K}_{ur}(0) \bm{Q}_r\\
           \bm{Q}_r^T \bm{K}_{ru}(0) \bm{Q}_u & \bm{\Lambda_r}
           \end{bmatrix} :=\bm{\Tilde{\Lambda}}.
\end{align}
Furthermore, letting $\bm{B} =   \bm{Q}_r^T \bm{K}_{ru}(0) \bm{Q}_u$ and $ \bm{P} = \begin{bmatrix}
           \bm{I} & 0 \\
           - \bm{B} \bm{\Lambda}_u^{-1}  & \bm{I}
            \end{bmatrix}$,
gives
\begin{align}
    \bm{\Tilde{\Lambda}} = \bm{P}^T    \begin{bmatrix}
           \bm{ \Lambda}_u & 0\\
           0& \bm{\Lambda}_r - \bm{B}^T \bm{\Lambda}_u^{-1} \bm{B}
           \end{bmatrix} \bm{P} =  \bm{P}^T   \bm{\Lambda} \bm{P}.
\end{align}
Therefore, we obtain
\begin{align}
\bm{Q}^T 
     \left( \begin{bmatrix}
    \mathcal{B}[\bm{u}](\bm{x}_b, {\bm \theta}(t))\\
   \mathcal{N}[\bm{u}](\bm{x}_r, {\bm \theta}(t))
    \end{bmatrix} -
     \begin{bmatrix}
     \bm{g}(\bm{x}_b) \\
    \bm{f}(\bm{x}_r)
    \end{bmatrix} \right) 
    \approx e^{- \bm{P }^T \bm{\Lambda}  \bm{P } t}\bm{Q}^T 
     \begin{bmatrix}
      \bm{g}(\bm{x}_b) \\
       \bm{f}(\bm{x}_r)
       \end{bmatrix}.
\end{align}
This concludes the proof.

\end{proof}